\documentclass[twoside,11pt]{article}

\usepackage{jmlr2e}

\firstpageno{1}

\usepackage{amsmath,amssymb,amsfonts}
\newtheorem{thm}{Theorem}
\newtheorem{assump}{Assumption}

\usepackage{algorithm}
\usepackage{algorithmic}
\usepackage{xcolor}
\DeclareMathOperator*{\argminA}{arg\,min}    

\begin{document}

	\title{On the Global Convergence of Fitted Q-Iteration with Two-layer Neural Network Parametrization}
	
	\author{\name Mudit Gaur\email mgaur@purdue.eduu \\
		\name Mridul Agarwal\email agarw180@purdue.edu \\
		\name Vaneet Aggarwal\email{vaneet@purdue.edu}\\
		\addr Purdue University, West Lafayette IN 47907, USA
	}
\editor{}

\maketitle

	\begin{abstract} \label{abstract}
Deep Q-learning based algorithms have been applied successfully in many decision making problems, while their theoretical foundations are not as well understood. In this paper, we study a Fitted Q-Iteration with two-layer ReLU neural network parameterization, and find the sample complexity guarantees for the algorithm. Our approach estimates the Q-function in each iteration using a convex optimization problem. We show that this approach achieves a sample complexity of $\tilde{\mathcal{O}}(1/\epsilon^{2})$, which is order-optimal. This result holds for a countable state-spaces and does not require any assumptions such as a linear or low rank structure on the MDP. %
 \end{abstract}

	\section{Introduction} \label{introduction}

Reinforcement learning aims to maximize the cumulative rewards wherein an agent interacts with the system in a sequential manner, and has been used in many applications including games \cite{silver2017mastering,DBLP:journals/corr/abs-1708-04782,Haliem2021LearningMG}, robotics \cite{Maes03reinforcementlearning}, autonomous driving \cite{9351818}, ride-sharing \cite{8793143}, networking \cite{geng2020multi}, recommender systems \cite{NEURIPS2018_210f760a}, etc. One of the key class of algorithms used in reinforcement learning are those based on Q-learning. Due to large state space (possibly infinite), such algorithms are parameterized \cite{van2016deep}. However, with  general parametrization (e.g., neural network parametrization), sample complexity guarantees to achieve an $\epsilon$ gap to the optimal Q-values have not been fully understood. In this paper, we study this problem to provide the first sample complexity results for Q-learning based algorithm with neural network parametrization.

We note that most of the theoretical results for Q-learning based algorithms have been in tabular setups which assume finite states such as in \cite{NEURIPS2020_4eab60e5,NEURIPS2018_d3b1fb02}, or assume a linear function approximation \cite{carvalho20nips,NEURIPS2021_70a32110,pmlr-v162-chen22f,zhou2021provably} where the linear representation is assumed to be known. Although practically large state spaces are handled using deep neural networks due to the universal approximation theorem, the theoretical analysis with neural network parametrization of Q-learning is challenging as the training of neural network involves solving a non-convex optimization problem. Recently, the authors of \cite{pmlr-v120-yang20a} studied the problem of neural network parametrization for Q-learning based algorithms, and provided the asymptotic global convergence guarantees when the parameters of the neural network used to approximate the Q function are large. In this paper, we study the first results on sample complexity to achieve the global convergence of Q-learning based algorithm with neural network parametrization.

In this paper, we propose a  fitted Q-Iteration based algorithm, wherein the Q-function is parametrized by a two-layer ReLU network. The key for the two-layer ReLU network is that finding the optimal parameters for estimating the Q-function could be converted to a convex problem. In order to reduce the problem from non-convex to convex, one can replace the ReLU function with an equivalent diagonal matrix. Further, the minima of both the non-convex ReLU optimization and the equivalent convex optimization are same \cite{wang2021hidden,pmlr-v119-pilanci20a}.

We find the sample complexity for fitted Q-Iteration based algorithm with neural network parametrization. The gap between the learnt Q-function and the optimal Q-function arises from the parametrization error due to the 2-layer neural network, the error incurred due to imperfect reformulation of the neural network as a convex optimization problem as well as the error due to the random nature of the underlying MDP. This error reduces with increasing the number of iterations of the fitted Q-iteration algorithm, the number of iterations of the convex optimization step, the number of data points sampled at each step. We achieve an overall sample complexity of $\tilde{\mathcal{O}}\left(\frac{1}{(1-\gamma)^4\epsilon^{2}}\right)$. Our proof consists of expressing the error in estimation of the Q-function at a fixed iteration in terms of the Bellman error incurred at each successive iteration upto that point. We then express the Bellman error in terms of the previously mentioned components for the statistical error. We then upper bound the expectation of these error components with respect to a distribution of the state action space.

\section{Related Works} \label{Related Works}

\subsection{Fitted Q-Iteration}

The analysis of Q-learning based algorithms has been studied since the early 1990's \cite{Watkins1992,10.1023/A:1022693225949}, where it was shown that small errors in
the approximation of a task's optimal value function cannot produce arbitrarily bad performance when actions are selected by
a greedy policy.  Value Iteration algorithms, and its analogue Approximate Value iteration for large state spaces has been shown to have finite error bounds as \cite{puterman2014markov}. Similar analysis for finite state space has been studied in \cite{books/lib/BertsekasT96} and \cite{munos:inria-00124685}.

When the state space is large (possibly infinite),  the value function can not be updated at each state action pair at every iteration. Thus,  approximate value iteration algorithms are used that obtains samples from state action space at each iteration and estimates the action value function  by minimizing a loss which is a function of the sampled state action pairs.
This is the basis of the Fitted Q-Iteration, first explained in \cite{NIPS1994_ef50c335}. Instead of updating the Q-function at each step of the trajectory, it collects a batch of sample transitions and updates the Q-functions based on the collected samples and the existing estimate of the Q-function. The obtained samples could be using a generative model \cite{munos2003error,ghavamzadeh2008regularized}, or using the buffer memory \cite{kozakowski2022q,wang2020off}.

\subsection{Deep Neural Networks in Reinforcement Learning} \label{Deep Neural Networks in Reinforcement Learning}

Parametrization of Q-network is required for scaling the Q-learning algorithms to a large state space. Neural networks have been used previously to parameterize the Q-function \cite{tesauro1995temporal,xu2020finite,fujimoto2019off}. This approach, also called Deep Q-learning has been widely used in many reinforcement learning applications \cite{Yang2020ReinforcementLI,en15196920,9101333}. However, fundamental guarantees of Q-learning with such function approximation are limited, because the non-convexity of neural network makes the parameter optimization problem non-convex.

Even though the sample complexity of Q-learning based algorithms have not been widely studied for general parametrization, we note that such results have been studied widely for policy gradient based approaches \cite{agarwal2020optimality,wang2019neural,zhang2022convergence}.  The policy gradient approaches directly optimize the policy, while still having the challenge that the parametrization makes the optimization of parameters non-convex. This is resolved by assuming that the class of parametrization (e.g., neural networks) is rich in the sense that arbitrary policy could be approximated by the parametrization. However, note that for systems with infinite state spaces, results for policy gradient methods consist of upper bounds on the gradient of the estimate of the average reward function, such as in \cite{yuan2022general}. For upper bounds on the error of estimation of the value function, linear structure on the MDP has to be assumed, as is done in \cite{chen2022finite}.

In the case of analysis of Fitted Q-Iteration (FQI) algorithms such as \cite{pmlr-v120-yang20a}, upper bounds on the error of estimation of the Q-function are obtained by considering sparse Neural networks with ReLU functions and Holder smooth assumptions on the Neural networks. At each iteration of the FQI algorithm, an estimate of the Q-function is obtained by optimizing a square loss error, which is non-convex in the parameters of the neural network used to represent the Q-function. Due to this non-convexity, the upper bounds are asymptotic in terms of the parameters of the neural network and certain unknown constants. \cite{xu2020finite} improves upon this result by demonstrating finite time error bounds for Q-learning with neural network parametrization. However the  error bounds obtained can go unbounded as the number of iterations increase (See Appendix \ref{finite_time_q_learning}), hence they do not give any sample complexity results.
Our result establishes the first sample complexity results for a (possibly) infinite state space without the need for a linear structure on the MDP.

\subsection{Neural Networks Parameter Estimation} \label{Neural Networks Parameter Estimation}

The global optimization of neural networks has shown to be NP hard \cite{BlumRivest:92}. Even though Stochastic Gradient Descent algorithms can be tuned to give highly accurate results as in \cite{bengio2012practical}, convergence analysis of such methods requires assumptions such as infinite width limit such as in \cite{zhu2021one}. Recently, it has been shown that the parameter optimization for two-layer ReLU neural networks can be converted to an equivalent convex program which is exactly solvable and computationally tractable \cite{pmlr-v119-pilanci20a}. Convex formulations for convolutions and deeper models have also been studied  \cite{sahiner2020vector,sahiner2020convex}. In this paper, we will use these approaches for estimating the parameterized Q-function.

\section{Problem Setup} \label{Problem Setup}

We study a discounted Markov Decision Process (MDP), which is described by the tuple ${\mathcal{M}}=(\mathcal{S}, \mathcal{A}, P, R, \gamma)$, where $\mathcal{S}$ is a bounded measurable state space, $\mathcal{A}$ is the finite set of actions, we represent set of state action pairs as $\left[0,1\right]^{d}$ (a $d$ dimensional tuple with all elements in [0,1]), where $d$ is a positive integer greater than 1. ${P}: \mathcal{S}\times\mathcal{A} \rightarrow \mathcal{P}(\mathcal{S})$  is the probability transition kernel\footnote{
For a measurable set $\mathcal{X}$, let $\mathcal{P}(\mathcal{X})$ denote the set of all probability measures over $\mathcal{X}$.
}.
${R}: \mathcal{S}\times\mathcal{A} \rightarrow \mathcal{P}([0,R_{\max}])$ is the reward kernel on the state action space with $R_{\max}$ being the absolute value of the maximum reward and $0<\gamma<1$ is the discount factor. A policy $\pi:\mathcal{S} \rightarrow \mathcal{P}(\mathcal{A})$ is a mapping that maps state to a probability distribution over the action space. 
 Here, we denote by $\mathcal{P}(\mathcal{S}), \mathcal{P}(\mathcal{A}), \mathcal{P}([a,b])$, the set of all probability distributions over the state space, the action space, and a closed interval $[a,b]$, respectively. We define the action value function %
 for a given policy $\pi$ respectively as follows. 
\begin{equation}
Q^{\pi}(s,a) = \mathbb{E}\left[\sum_{t=0}^{\infty}\gamma^{t}r'(s_{t},a_{t})|s_{0}=s,a_{0}=a\right] \label{ps_1},
\end{equation}
where $r'(s_{t},a_{t}) \sim R(\cdot|s_{t},a_{t})$, $a_{t+1} \sim \pi(\cdot|s_{t+1})$ and $s_{t+1} \sim P(\cdot|s_{t},a_{t})$ for $t=\{0,\cdots,\infty\}$. For a discounted MDP, we  define the optimal
action value functions as follows:
\begin{equation}
Q^{*}(s,a) =  \sup_{\pi}Q^{\pi}(s,a)  \hspace{0.5cm} \forall (s,a) \in  \mathcal{S}\times\mathcal{A} \label{ps_2},
\end{equation}

A policy that achieves the optimal  action value functions is known as the  \textit{optimal policy} and is denoted as $\pi^{*}$. It can be shown that $\pi^{*}$ is the greedy policy with respect to $Q^{*}$ \cite{10.5555/1512940}. Hence finding  $Q^{*}$ is sufficient to obtain the optimal policy. We define the Bellman operator for a policy $\pi$ as follows
\begin{equation}
(T^{\pi}Q)(s,a) =  r(s,a) + \gamma \int Q^{\pi}(s',\pi(s'))P(ds'|s,a), \label{ps_3}
\end{equation}
where $r(s,a) = \mathbb{E}(r'(s,a)|(s,a))$ Similarly we define the Bellman Optimality Operator as  
\begin{equation}
(TQ)(s,a) = \left(r(s,a) +\max_{a' \in \mathcal{A}}\gamma\int Q(s',a')P(ds'|s,a)\right),\label{ps_4}
\end{equation}

Further, operator $P^{\pi}$ is  defined as
\begin{equation}
P^{\pi}Q(s,a)=\mathbb{E}[Q(s',a')|s' \sim P(\cdot|s,a), a' \sim \pi(\cdot|s') ]\label{ps_7} ,    
\end{equation}
which is the one step Markov transition operator for policy $\pi$ for the Markov chain defined on $\mathcal{S}\times\mathcal{A}$ with the transition dynamics given by $S_{t+1} \sim P(\cdot|S_{t},A_{t})$ and $A_{t+1} \sim \pi(\cdot|S_{t+1})$. It defines a distribution on the state action space after one transition from the initial state. Similarly, $P^{\pi_{1}}P^{\pi_{2}}\cdots{P}^{\pi_{m}}$ is the $m$-step Markov transition operator following policy $\pi_t$ at steps $1\le t\le m$. It defines a distribution on the state action space after $m$ transitions from the initial state. We have the relation 
\begin{eqnarray}
(T^{\pi}Q)(s,a) &=&  r(s,a) + \gamma \int Q^{\pi}(s',\pi(s'))P(ds'|s,a) \nonumber\\
         &=&  r(s,a) + {\gamma}(P^{\pi}Q)(s,a)
\end{eqnarray}

Which defines $P^{*}$ as 
\begin{equation}
P^{*}Q(s,a)=\max_{a^{'} \in \mathcal{A}}\mathbb{E}[Q(s',a')|s' \sim P(\cdot|s,a)]\label{ps_7_3} ,
\end{equation}
in other words, $P^{*}$ is the one step Markov transition operator with respect to the greedy policy of the function on which it is acting.

This gives us the relation 
\begin{equation}
(TQ)(s,a) = r(s,a) +{\gamma}P^{*}Q(s,a) ,\label{ps_7_3_1}\\
\end{equation}

For any measurable function $f:\mathcal{S}\times\mathcal{A}:\rightarrow\mathbb{R}$, we also define 
\begin{equation}
\mathbb{E}(f)_{\nu}=\int_{\mathcal{S}\times\mathcal{A}}fd\nu,  
\end{equation}
for any distribution $\nu\in\mathcal{P}(\mathcal{S}\times\mathcal{A})$.

For representing the action value function, we will use a 2 layer ReLU neural network. A 2-layer ReLU Neural Network with input $x \in \mathbb{R}^{d}$ is defined as 
\begin{equation}
    f(x) = \sum_{i=1}^{m}\sigma'(x^{T}u_{i})\alpha_{i}, \label{ReLU_0}
\end{equation}
where $m \geq 1$ is the number of neurons in the neural network, the parameter space is $\Theta_{m}=\mathbb{R}^{d\times m} \times \mathbb{R}^{m}$ and  $\theta =(U,\alpha)$ is an element of the parameter space, where $u_{i}$ is the $i^{th}$ column of $U$, and $\alpha_{i}$ is the $i^{th}$ coefficient of $\alpha$. The function $\sigma': \mathbb{R} \rightarrow \mathbb{R}_{\ge 0}$ is the ReLU or restricted linear function defined as     $\sigma'(x) \triangleq\max(x,0)$. In order to obtain parameter $\theta$ for a given set of data $X \in \mathbb{R}^{n \times d}$  and the corresponding response values $y \in \mathbb{R}^{n \times 1}$, we desire the parameter that minimizes the squared loss  (with a regularization parameter $\beta \in [0,1]$), given by 
\begin{eqnarray}
\mathcal{L}(\theta) &=& \argminA_{\theta} ||\sum_{i=1}^{m}\sigma(Xu_{i})\alpha_{i}- y||_{2}^{2} + \beta  \sum_{i=1}^{m}||u_{i}||_{2}^{2} \nonumber\\
                    &&    + \beta \sum_{i=1}^{m}|\alpha_{i}|^{2} \label{ReLU_1}
\end{eqnarray}

Here, we have the term $\sigma(Xu_{i})$ which is a vector $\{\sigma'((x_{j})^{T}u_{i})\}_{j \in \{1,\cdots, n\}}$ where $x_{j}$ is the $j^{th}$ row of $X$. It is the ReLU function applied to each element of the vector $Xu_{i}$. We note that the optimization in Equation \eqref{ReLU_1} is non-convex in $\theta$ due to the presence of the ReLU activation function. In \cite{wang2021hidden}, it is shown that this optimization problem has  an equivalent convex form, provided that the number of neurons $m$ goes above a certain threshold value. This convex problem is obtained by replacing the ReLU functions in the optimization problem with equivalent diagonal operators.
The convex problem is given as
\begin{eqnarray}
     \mathcal{L}^{'}_{\beta}(p) &:=& ( \argminA_{p} ||\sum_{D_{i} \in D_{X}}D_{i}(Xp_{i}) - y||^{2}_{2} \nonumber\\
     && + \beta \sum_{D_{i} \in D_{X}} ||p_{i}||_{2}) ,\label{ReLU_1_0}
\end{eqnarray}
where $p \in \mathbb{R}^{d \times |D_{X}|}$. 

$D_{X}$ is the set of diagonal matrices $D_{i}$ which depend on the dataset $X$. Except for cases of $X$ being low rank it is not computationally feasible to obtain the set $D_{X}$. We instead use $\tilde{D} \in D_{X}$ to solve the convex problem

\begin{eqnarray}
     \mathcal{L}^{'}_{\beta}(p) &:=& ( \argminA_{p} ||\sum_{D_{i} \in \tilde{D}}D_{i}(Xp_{i}) - y||^{2}_{2} \nonumber\\
     && + \beta \sum_{D_{i} \in \tilde{D}} ||p_{i}||_{2}) \label{ReLU_1_1}
\end{eqnarray}
where $p \in \mathbb{R}^{d \times |\tilde{D}|}$. 

The relevant details of the formulation and the definition of the diagonal matrices $D_{i}$  are provided in Appendix \ref{cones_apdx}. For a set of parameters $\theta =(u,\alpha) \in \Theta$, we denote neural network represented by these parameters as 
\begin{equation}
Q_{\theta}(x)=\sum_{i=1}^{m}\sigma'(x^{T}u_{i})\alpha_{i} \label{ReLU_1_2}
\end{equation}

\section{Proposed Algorithm} \label{Proposed Algorithm}

\begin{algorithm}
	\caption{Iterative algorithm to estimate $Q$ function}
	\label{algo_1}
	\textbf{Input:} $\mathcal{S},$ $ \mathcal{A}, \gamma, $ Time Horizon K $ \in \mathcal{Z},$ sampling 
                         distribution $\nu$, one step transition operator $\kappa$, $T_{k: k \in \{1,\cdots,K\}}$\\
        \textbf{Initialize:} $\tilde{Q}(s,a)=0 \hspace{0.1cm}\forall (s,a) \in \mathcal{S}\times\mathcal{A}$ \\
	\begin{algorithmic}[1]
		\FOR{$k\in\{1,\cdots,K\}$} 
		{   
		    \STATE sample $n$ i.i.d ({$s_{i},a_{i}$}) with $s_{i},a_{i}$ drawn from the sampling distribution $\nu \label{a1_l1}$\\
                \STATE obtain $\{s_{i}',r'_{i}\}$ from  $\kappa(s_{i},a_{i})=\{s_{i}',r'_{i}\}$        
                \STATE Set $y_i= r'_{i} + \gamma\max_{a' \in \mathcal{A}}\tilde{Q}(s_{i}',a')$, where $i \in \{1,\cdots,n\}$ \label{a1_l2}\\ 
                \STATE Set $X_{k},Y_{k}$ as  the matrix of the sampled state action pairs and vector of estimated $Q$ values respectively \label{a1_l3}\\
		      \STATE \textbf{Call  Algorithm } \ref{algo_2} with input ($X=X_{k}$, $y=Y_{k}$, $T=T_{k}$) and  return parameter $\theta$ \label{a1_l4}\\
		    \STATE $\tilde{Q}=Q_{\theta}$
		}
		\ENDFOR\\
	Define $\pi_{K}$ as the greedy policy with respect to $\tilde{Q}$\\
	Output: An estimator $\tilde{Q}$ of $Q^{*}$ and $\pi_{k}$ as it's greedy policy 
	\end{algorithmic}
\end{algorithm}

In this section, we describe our Neural Network Fitted Q-iteration algorithm. The key in the algorithm is the use of convex optimization for the update of parameters of the Q-function. The algorithm, at each iteration $k$, updates the estimate of the $Q$ function, here denoted as $Q_{k}$. The update at each step in the ideal case is to be done by applying the Bellman optimality operator defined in Equation \eqref{ps_3}. However, there are two issues which prevent us from doing that. First, we do not know the transition kernel $P$. Second for the case of an infinite state space, we cannot update the estimate at each iteration of the state space. Therefore, we apply an approximation of the Bellman optimality operator defined as 

\begin{equation}
 \hat{T}Q(s,a) = \left(r'(s,a) +\max_{a' \in \mathcal{A}}\gamma Q(s',a')\right), \label{pa_1}
\end{equation}

where  $r'(s,a) \sim {R}(.|s,a)$.  
Since we cannot perform even this approximated optimality operator for all state action pairs due to the possibly infinite state space, we instead update our estimate of the $Q$ function at iteration $k$ as the 2 layer ReLU Neural Network which minimizes the following loss function

\begin{equation}
   \frac{1}{n} \sum_{i=1}^{n} \left( Q_{\theta}(s_{i},a_{i}) - \left(r'(s_{i},a_{i}) + \max_{a' \in \mathcal{A}}\gamma Q_{k-1}(s',a') \right) \right)^{2}, \label{pa_2}
\end{equation}

Here $Q_{k-1}$ is the estimate of the $Q$ function from iteration $k-1$ and the state action pairs have been sampled from some distribution, $\nu$, over the state action pairs. Note that this is a problem of the form given in Equation \eqref{ReLU_1} with $y_{i} =  \left(r'(s_{i},a_{i}) + \max_{a' \in \mathcal{A}}\gamma Q_{k-1}(s',a') \right)$ where $i \in (1,\cdots,n)$ and $Q_{\theta}$ represented as in Equation \eqref{ReLU_1_2}.

We define the Bellman error at iteration $k$ as

\begin{equation}
    \epsilon_{k} = TQ_{k-1} - Q_{k} ,\label{pa_2_1}
\end{equation}

The main algorithm, Algorithm \ref{algo_1}, iteratively samples from the state action space at every iteration, and the corresponding reward is observed. For each of the sampled state action pairs, the approximate Bellman optimality operator given in Equation \eqref{pa_1} is applied to obtain the corresponding output value $y$. We have access to these values under Assumption \ref{assump_4}. The sampled set of state action pairs and the corresponding $y$ values are then passed to Algorithm \ref{algo_2}, which returns the estimates of the neural network which minimizes the loss given in Equation \eqref{pa_2}. The algorithm updates the estimate of the action value function as the neural network corresponding to the estimated parameters. 

\begin{algorithm}
	\caption{Neural Network Parameter Estimation}     %
	\label{algo_2}
	\begin{algorithmic}[1]
	\STATE{\textbf{Input:}} data $(X,y,T)$ \\
	\STATE{\textbf{Sample:}} $\tilde{D}={diag(1(Xg_{i}>0))} : g_{i} \sim \mathcal{N}(0,I), i \in [|\tilde{D}|]  $ \label{a2_l1}\\
	\STATE \textbf{Initialize} $y^1=0,  u^1=0$ \\
        \textbf{Initialize} $g(u) = || \sum_{D_{i} \in \tilde{D}}D_{i}Xu_{i}-y ||^{2}_{2}$ 
        \FOR{$k\in\{0,\cdots,T\}$} \label{a2_l2}
		{
            \STATE $u^{k+1}= y_{k} - \alpha_{k}\nabla{g(y_{k})} $\label{a2_l3}\\
            \STATE $y^{k+1} = \argminA_{y:|y|_{1} \le \frac{R_{max}}{1-\gamma}} ||u_{k+1}-y||^{2}_{2}$ \label{a2_l4}\\
		}
		\ENDFOR\\
        \STATE Set $u^{T+1}=u^{*}$
	\STATE{Solve Cone Decomposition:}\\ 
	$\bar{v}, \bar{w} \in {u_{i}^{*}=v_{i}-w_{i}, i \in [d]\} }$ such that $v_{i},w_{i} \in \mathcal{K}_{i}$ and at-least one $v_{i},w_{i}$ is zero. \label{a2_l5}\\
	\STATE Construct $(\theta=\{u_{i},\alpha_{i}\})$  using the transformation
         \begin{eqnarray} 
            \psi(v_{i},w_{i}) &=& \left\{\begin{array}{lr}({v}_{i},1), & \text{if } {w}_{i}=0 \label{alg_2_trans}\\   
            ({w}_{i},-1), & \text{if }  
             {v}_{i} = 0\\
            (0,0), & \text{if } {v}_{i} = {w}_{i} = 0  \end{array} \right.
        \end{eqnarray}
        for all  ${i \in \{ 1,\cdots,m\}}$  
         \label{a2_l6}\\
        \STATE Return $\theta$ \label{a2_l7}\\
   \end{algorithmic}
\end{algorithm}

Algorithm \ref{algo_2} optimizes the parameters for the neural network at each step of Algorithm \ref{algo_1}. This is performed by reducing the problem to an equivalent convex problem as described in Appendix \ref{cones_apdx}. The algorithm  first  samples a set of diagonal matrices denoted by $\tilde{D}$ in line \ref{a2_l1} of Algorithm \ref{algo_2}. The elements of $\tilde{D}$ act as the diagonal matrix replacement of the ReLU function. Algorithm \ref{algo_2} then solves an optimization of the form given in Equation \eqref{ReLU_1_1} with the regularization parameter $\beta=0$. This convex optimization is solved in Algorithm \ref{algo_2} using the projected gradient descent algorithm. After obtaining the optima for this convex program, denoted by $u^{*}=\{u^{*}_{i}\}_{i \in \{1,\cdots,|\tilde{D}|\}}$, in line \ref{a2_l6}, we transform them to the parameters of a neural network of the form given in Equation \eqref{ReLU_1_2} which are then passed back to Algorithm \ref{algo_1}. The procedure is described in detail along with the relevant definitions in Appendix \ref{cones_apdx}.

	\section{Error Characterization}
We now characterize the errors which can result in the gap between the point of convergence and the optimal Q-function. To define the errors, we first define the various possible $Q$-functions which we can approximate in decreasing order of the accuracy.

We start by defining the best possible Q-function, $Q_{k_1}$ for episode $k > 1$. $Q_{k_{1}}$ is the best approximation of the function $TQ_{k-1}$ possible from the class of two layer ReLU neural networks, with respect to the expected square from the true ground truth $TQ_{k-1}$.
\begin{definition} \label{def_1}
   For a given iteration $k$ of Algorithm \ref{algo_1}, we define
   \begin{equation}
   Q_{k_{1}}=\argminA_{Q_{\theta},\theta \in \Theta}\mathbb{E}(Q_{\theta} - TQ_{k-1})^{2}_{\nu},    
   \end{equation}
\end{definition}
The expectation is with respect to the sampling distribution of the state action pairs denoted by $\nu$. $TQ_{k-1}$ is the bellman optimality operator applied to $Q_{k-1}$.

Note that we do not have access to the transition probability kernel $P$, hence we cannot calculate $TQ_{k-1}$. To alleviate this, we use the observed next states to estimate the Q-value function. Using this, we define  $Q_{k_2}$ as,. 

\begin{definition} \label{def_2}
   For a given function $Q:\mathcal{S}\times\mathcal{A} \rightarrow \left[0,\frac{R_{max}}{1-\gamma}\right]$, 
   we define
   \begin{eqnarray}
    Q_{k_2}=\argminA_{Q_{\theta},\theta \in \Theta}\mathbb{E}_{(s,a) \sim \nu ,s' \sim P(s'|s,a), r'(\cdot|s,a)\sim {R}(\cdot|s,a)}\nonumber\\
    (Q_{\theta}(s,a)-(r'(s,a)+\gamma\max_{a'}Q_{k-1}(s',a'))^{2}  ,  \label{temp} 
   \end{eqnarray}
\end{definition}
Compared to $Q_{k_1}$, in $Q_{k_2}$, we are minimizing the expected square loss from target function $\big(r'(s,a)+\gamma\max_{a'}Q(s',a')\big)$ or the expected loss function. %

To obtain $Q_{k_2}$, we still need to compute the true expected value in Equation \ref{temp}. However, we still do not know the transition function $P$. To remove this limitation, we use sampling. Consider a set, $\mathcal{X}$ , of state-action pairs sampled using distribution $\nu$. We now define $Q_{k_3}$ as,
\begin{definition} \label{def_3}
   For a given set of state action pairs $\mathcal{X}$ and a given function $Q:\mathcal{S}\times\mathcal{A} \rightarrow \left[0,\frac{R_{max}}{1-\gamma}\right]$ we define
   \begin{eqnarray}
    Q_{k_3}=\argminA_{Q_{\theta},\theta \in \Theta}\frac{1}{|\mathcal{X}|} \sum_{(s_{i},a_{i}) \in \mathcal{X}}\Big( Q_{\theta}(s_{i},a_{i}) \nonumber\\ 
        - \big(r'(s_{i},a_{i}) + \gamma\max_{a' \in \mathcal{A}}Q_{k-1}(s'_i,a') \big)\Big)^{2},
   \end{eqnarray}
where $r'(s_{i},a_{i})$, and $s_i'$ are the observes reward and the observed next state for state action pair $s_i, a_i$ respectively.
\end{definition}

$Q_{k_3}$ is the best possible approximation for $Q$-value function which minimizes the sample average of the square loss functions with the target values as $ \big(r'(s_{i},a_{i})+ \gamma\max_{a' \in \mathcal{A}}Q(s',a') \big)^{2}$ or the empirical loss function. %

After defining the possible solutions for the $Q$-values using different loss functions, we define the errors.

We first define approximation error which represents the difference between $TQ_{k-1}$ and its best approximation possible from the class of 2 layer ReLU neural networks. We have
\begin{definition}[Approximation Error] \label{def_4}
   For a given iteration $k$ of Algorithm \ref{algo_1}, $\epsilon_{k_{1}} =TQ_{k-1} - Q_{k_{1}}$, where $Q_{k-1}$ is the estimate of the $Q$ function at the iteration $k-1$.
\end{definition}

We also define Estimation Error which denotes the error between the best approximation of $TQ_{k-1}$ possible from a 2 layer ReLU neural network and $Q_{k_2}$. We demonstrate that these two terms are the same and this error is zero.
\begin{definition}[Estimation Error] \label{def_5}
   For a given iteration $k$ of Algorithm \ref{algo_1}, $\epsilon_{k_{2}} =Q_{k_{1}} - Q_{k_{2}}$.
\end{definition}

We now define Sampling error which denotes the difference between the minimizer of expected loss function $Q_{k_2}$ and the minimizer of the empirical loss function using samples, $Q_{k_3}$. We will use Rademacher complexity results to upper bound this error.
\begin{definition}[Sampling Error] \label{def_6}
   For a given iteration $k$ of Algorithm \ref{algo_1}, $\epsilon_{k_{3}} =Q_{k_{2}} - Q_{k_{3}}$. Here $X_{k}$ is the set of state action pairs sampled at the $k^{th}$ iteration of Algorithm \ref{algo_1}. 
\end{definition}

Lastly, we define optimization error which denotes the difference between the minimizer of the empirical square loss function, $Q_{k_3}$, and our estimate of this minimizer that is obtained from the projected gradient descent algorithm.
\begin{definition}[Optimization Error] \label{def_7}
   For a given iteration $k$ of Algorithm \ref{algo_1}, $\epsilon_{k_{4}} =Q_{k_{3}}- Q_{k}$. Here $Q_{k}$ is our estimate of the $Q$ function at iteration $k$ of Algorithm \ref{algo_1}. 
\end{definition}
	\section{Assumptions} \label{Assumptions}
In this section, we formally describe the assumptions that will be used in the results. \\

\begin{assump}  \label{assump_2} 
Let $\theta^{*} \triangleq \arg\min_{\theta \in \Theta} \mathcal{L}(\theta)$, where $\mathcal{L}(\theta)$ is defined in  \eqref{ReLU_1} and we denote $Q_{\theta^{*}}(\cdot)$ as $Q_{\theta}(\cdot)$ as defined in \eqref{ReLU_1_2} for $\theta=\theta^{*}$. 
Also, let $\theta_{\tilde{D}}^{*} \triangleq \arg\min_{\theta \in \Theta} \mathcal{L}_{|\tilde{D}|}(\theta)$, where $\mathcal{L}_{\tilde{D}}(\theta)$ is defined in  \eqref{ReLU_6}. Further, we denote $Q_{\theta_{|\tilde{D}|}^{*}}(\cdot)$ as $Q_{\theta}(\cdot)$ as defined in \eqref{ReLU_1_2} for $\theta=\theta_{|\tilde{D}|}^{*}$. Then we assume
\begin{equation}
    \mathbb{E}(|Q_{\theta^{*}} - Q_{\theta_{|\tilde{D}|}^{*}}|)_{\nu} \leq \epsilon_{|\tilde{D}|},
\end{equation}
for any $\nu \in \mathcal{P}(\mathcal{S}\times\mathcal{A})$
\end{assump}

$\mathcal{L}_{\tilde{D}}(\theta)$ is the non-convex problem equivalent to the convex problem in \eqref{ReLU_1_1}. Thus, $\epsilon_{|\tilde{D}|}$ is a measure of the error incurred due to taking a sample of diagonal matrices $\tilde{D}$ and not the full set $D_{X}$. In practice, setting $|\tilde{D}|$ to be the same order of magnitude as $d$  (dimension of the data) gives us a sufficient number of diagonal matrices to get a reformulation of the non convex optimization problem which performs comparably or better than existing gradient descent algorithms, therefore $\epsilon_{|\tilde{D}|}$ is only included for theoretical completeness and will be negligible in practice. This has been practically demonstrated in \cite{pmlr-v162-mishkin22a,pmlr-v162-bartan22a,pmlr-v162-sahiner22a}. Refer to Appendix \ref{cones_apdx} for details of $D_{X}$, $\tilde{D}$ and $\mathcal{L}_{|\tilde{D}|}(\theta)$.

\begin{assump}  \label{assump_3} 
We assume that for all functions $Q:\mathcal{S}\times\mathcal{A} \rightarrow \left[0,\left(\frac{R_{\max}}{1-\gamma}\right)\right]$,  there exists a function $ Q_{\theta} $ where $ \theta \in \Theta $ such that 
\begin{eqnarray}
\mathbb{E}{(Q_{\theta}-Q)}^{2}_{\nu} \le \epsilon_{bias}, 
\end{eqnarray}
for any $\nu \in \mathcal{P}(\mathcal{S}\times\mathcal{A})$. 
\end{assump}

$\epsilon_{bias}$ reflects the error that is incurred due to the inherent lack of expressiveness of the neural network function class. In the analysis of \cite{pmlr-v120-yang20a}, this error is assumed to be zero. We account for this error with an assumption similar to the one used in \cite{NEURIPS2020_56577889}. %

\begin{assump}  \label{assump_4} 
We assume that for the MDP $\mathcal{M}$, we have access to a one step transition operator $\kappa : \mathcal{S}\times\mathcal{A} \rightarrow \mathcal{S}\times\left([0,R_{max}]\right)$ defined as 
\begin{equation}
    \kappa(s,a) = \left(s',r'(s,a)\right),
\end{equation}
where $s' \sim \mathcal{P}(\cdot|(s,a))$ and $r'(s,a) \sim {R}(.|s,a)$
\end{assump}

One step sampling distributions offer a useful tool for finding next state and rewards as i.i.d samples. In practice such an operator may not be available and we only have access to these samples as a Markov chain. This can be overcome be storing the transitions and then sampling the stored transitions independently in a technique known as experience replay. Examples of experience replay being used are \cite{mnih2013playing,agarwal2021online,NIPS2017_453fadbd}.

\begin{assump}  \label{assump_6} 
Let $\nu_{1}$ be a probability measure on $\mathcal{S}\times\mathcal{A}$ which is absolutely continuous with respect to the Lebesgue measure. Let $\{\pi_{t}\}$ be a sequence of policies and suppose that the state action pair has an initial distribution of $\nu_{1}$. Then we assume that for all $\nu_{1}, \nu_{2} \in \mathcal{P}(\mathcal{S}\times\mathcal{A})$ there exists a constant $\phi_{\nu_{1},\nu_{2}} \le \infty$ such that
\begin{eqnarray}
    \sup_{\pi_{1},\pi_{2},\cdots,\pi_{m}}\Bigg|\Bigg| \frac{d(P^{\pi_{1}}P^{\pi_{2}}\cdots{P}^{\pi_{m}}\nu_{2})}{d\nu_{1}}\Bigg|\Bigg|_{\infty} &\le& \phi_{\nu_{1},\nu_{2} } \nonumber\\
    \label{assump_6_1}
\end{eqnarray}
for all $m \in \{1,\cdots,\infty\}$, where  $ \frac{d(P^{\pi_{1}}P^{\pi_{2}}\cdots{P}^{\pi_{m}}\nu_{2})}{d\nu_{1}}$ denotes the Radon Nikodym derivative of the state action distribution $P^{\pi_{1}}P^{\pi_{2}}\cdots{P}^{\pi_{m}}\nu_{2}$ with respect to the distribution $\nu_{1}$.
\end{assump}

 This assumption puts an upper bound on the difference between the state action distribution $\nu_{1}$ and the state action distribution induced by sampling a state action pair from the distribution $\mu_{2}$ followed by any possible policy for the next $m$ steps for any finite value of $m$. Similar assumptions have been made in \cite{pmlr-v120-yang20a,JMLR:v17:10-364}. 
	\section{Supporting Lemmas}
We will now state the key lemmas that will be used for finding the sample complexity of the proposed algorithm. 

\begin{lemma} \label{lem_1}
    For any given iteration $k \in \{1,\cdots,K\}$ for the approximation error denoted by $\epsilon_{k_{1}}$ in Definition \ref{def_4}, we have 
    \begin{equation}
        \mathbb{E}\left(|\epsilon_{k_{1}}|\right)_{\nu} \le \sqrt{\epsilon_{bias}},
    \end{equation}
\end{lemma}

\textit{Proof Sketch:} We use Assumption \ref{assump_3} and the definition of the variance of a random variable to obtain the required result. The detailed proof is given in Appendix \ref{proof_lem_1}.

\begin{lemma} \label{lem_2}
     For any given iteration $k \in \{1,\cdots,K\}$,  $Q_{k_1}=Q_{k_2}$, or equivalently $\epsilon_{k_2}=0$
\end{lemma}

\textit{Proof Sketch:} We use  Lemma \ref{sup_lem_1} in Appendix \ref{sup_lem} and use the definitions of $Q_{k_1}$ and $Q_{k_2}$ to prove this result. The detailed proof is given in Appendix \ref{proof_lem_2}.

\begin{lemma} \label{lem_3}
    For any given iteration $k \in \{1,\cdots,K\}$, if the number of samples of the state action pairs sampled by Algorithm \ref{algo_1} at iteration $k$, denoted by $n_{k}$, satisfies
    \begin{eqnarray}
    n_{k} &\ge& 8\left({(C_{k}{\eta}\beta_{k}})^{2}{\epsilon}^{-2}\right) 
    \end{eqnarray}
    for some constants $C_{k}, \eta$ and $\beta_{k}$, then the error $\epsilon_{k_{3}}$ defined in Definition \ref{def_6} is upper bounded as
    \begin{equation}
        \mathbb{E}\left(|\epsilon_{k_{3}}|\right)_{\nu} \le  \epsilon, 
    \end{equation}
\end{lemma}

\textit{Proof Sketch:} First we note that for a fixed iteration $k$ of Algorithm \ref{algo_1},  $\mathbb{E}(R_{X,Q_{k-1}}({\theta})) = L_{Q_{k-1}}({\theta})$ where $R_{X,Q_{k-1}}({\theta})$ and $L_{Q_{k-1}}({\theta})$ are defined in Appendix \ref{proof_lem_3}. We use this to get a probabilistic bound on the expected value of $|(Q_{k_2}) - (Q_{k_3})|$ using Rademacher complexity theory. The detailed proof is given in Appendix \ref{proof_lem_3}.

\begin{lemma} \label{lem_4}
     For any given iteration $k \in \{1,\cdots,K\}$ of Algorithm \ref{algo_1}, let the number of steps of the projected gradient descent performed by Algorithm \ref{algo_2}, denoted by $T_{k}$, and the gradient descent step size $\alpha_{k}$ satisfy
    \begin{eqnarray}
      T_{k} &\ge& \left({\frac{C^{'}_{k}\epsilon}{l_{k}}}\right)^{-2}L_{k}^{2}||u_{k}^{*}||^{2}_{2} -1\\
      \alpha_{k} &=& \frac{||u^{*}_{k}||_{2}}{L_{k}\sqrt{T_{k}+1}}
    \end{eqnarray}
   where $\epsilon < C^{'}_{k}$, for some constants $C^{'}_{k}, l_{k}$, $L_{k}$ and  $||\left(u_{k}^{*}\right)||_{2}$. Then the error $\epsilon_{k_{4}}$ defined in Definition \ref{def_7} is upper bounded as
    \begin{equation}
       \mathbb{E}(|\epsilon_{k_{4}}|)_{\nu} \le \epsilon + {\epsilon_{|\tilde{D}|}},
    \end{equation}
\end{lemma}

\textit{Proof Sketch:} We use the number of iterations $T_{k}$ required to get an $\epsilon$ bound on the difference between the minimum objective value and the objective value corresponding to the estimated parameter at iteration $T_{k}$. We use the convexity of the objective and the Lipschitz property of the neural network to get a bound on the $Q$ functions corresponding to the estimated parameters. The detailed proof is given in Appendix \ref{proof_lem_4}.
	
\section{Main Result} \label{Main Result}
In this section, we provide the guarantees for the proposed algorithm, which is given in the following theorem. 

\begin{thm} \label{thm}
Suppose Assumptions \ref{assump_2}-\ref{assump_6} hold. Let Algorithm \ref{algo_1} run for $K$ iterations with $n_{k}$  state-action pairs sampled at iteration $k \in \{1,\cdots,K\}$, $T_{k}$ be the number of steps used and step size $\alpha_{k}$ in the projected gradient descent in Algorithm \ref{algo_2} at iteration $k \in \{1,\cdots,K\}$ and $|\tilde{D}|$ be the number of diagonal matrices sampled in Algorithm \ref{algo_2} for all iterations. Let $\nu \in \mathcal{P}(\mathcal{S}\times\mathcal{A})$ be the state action distribution used to sample the state action pairs in Algorithm \ref{algo_1}. Further, let $\epsilon \in (0,1)$, $C_{k}$, $C^{'}_{k}$, $l_{k}$, ,$L_{k}$, $\phi_{\nu,\mu}, \beta_{k}, \eta, \left(||u^{*}_{k}||_{2}\right)$ be constants. If we have, 

\begin{eqnarray}
   K &\geq& \frac{1}{\log\left(\frac{1}{\gamma}\right)}\log\left(\frac{6\phi_{\nu,\mu}R_{max}}{\epsilon(1-\gamma)^{2}}\right) -1  \\  
   n_{k} &\ge& 288(C_{k}{\eta}\beta_{k}\phi_{\nu,\mu}\gamma)^{2} (1-\gamma)^{-4}{\epsilon}^{-2} \\
   T_{k} &\ge& \left({\frac{C^{'}_{k}\epsilon(1-\gamma)^{2}}{6l_{k}\phi_{\nu,\mu}\gamma}}\right)^{-2}L_{k}^{2}||u_{k}^{*}||^{2}_{2} -1 \nonumber\\
    &&\\
    \alpha_{k} &=& \frac{||u^{*}_{k}||_{2}}{L_{k}\sqrt{T_{k}+1}}\\
    &&\forall k \in \{1,\cdots,K\} \nonumber
\end{eqnarray}

and $\epsilon < (C^{'}_{k})$ for all $k \in (1,\cdots,K)$, then we obtain

\begin{equation}
E{(Q^{*}-Q^{\pi_{K}})}_{\mu} \le \epsilon + \frac{2\phi_{\nu,\mu}\gamma}{(1-\gamma)^{2}}(\sqrt{\epsilon_{bias}} + {\epsilon_{|\tilde{D}|}}), \nonumber
\end{equation}

where $\mu \in \mathcal{P}(\mathcal{S}\times\mathcal{A})$. Thus we get an overall sample complexity given by  $\sum_{k=1}^{K}n_{k}=\tilde{\mathcal{O}}\left(\epsilon^{-2}(1-\gamma)^{-4}\right)$.
\end{thm}

The above algorithm achieves a sample complexity of $\tilde{\mathcal{O}}(1/\epsilon^{2})$, which is the first such result for general parametrized large state space reinforcement learning. Further, we note that $\tilde{\mathcal{O}}(1/\epsilon^{2})$ is the best order-optimal result even in tabular setup \cite{zhang2020almost}. In the following, we provide an outline of the proof, with the detailed proof provided in the Appendix \ref{thm proof}.

The expectation of the  difference between our estimated $Q$ function denoted by $Q^{\pi_{K}}$ and the optimal $Q$ function denoted by $Q^{*}$ (where $\pi_{k}$ is the policy obtained at the final step $K$ of algorithm \ref{algo_1}) is first expressed as a function of the Bellman errors (defined in Equation \eqref{pa_2_1}) incurred at each step of Algorithm \ref{algo_1}. The Bellman errors are then split into different components which are analysed separately. The proof is thus split into two stages. In the first stage, we demonstrate how the expectation of the error of estimation $(Q^{\pi_{K}}-Q^{*})$ of the $Q$ function is upper bounded by a function of the Bellman errors incurred till the final step $K$. The second part is to upper bound the expectation of the Bellman error.

{\bf Upper Bounding Q Error In Terms Of Bellman Error: } Since we only have access to the approximate Bellman optimality operator defined in Equation \eqref{pa_1}, we will rely upon the analysis laid out in \cite{farahmand2010error} and instead of the iteration of the value functions, we will apply a similar analysis to the action value function to get the desired result.  We recreate the result for the value function from Lemmas 2 and 3 of \cite{munos2003error}  for the action value function $Q$ to obtain  
\begin{eqnarray}
    Q^{*}-Q_{K} &\leq& \sum_{k=1}^{K-1} \gamma^{K-k-1} (P^{\pi^{*}})^{K-k-1}\epsilon_{k} , \label{po_1}\\\nonumber
    && + \gamma^{K} (P^{\pi^{*}})^{K}(Q^{*}-Q_{0}) 
 \end{eqnarray}

and 
 \begin{eqnarray}
 Q^{*}-Q^{\pi_{K}} &\leq& (I-\gamma P^{\pi_{K}})^{-1}\left(P^{\pi^{*}} - P^{\pi_{K}}\right) \nonumber\\
                   &&(Q^{*}-Q_{K})\label{po_2}
\end{eqnarray}
    where $\epsilon_{k} =TQ_{k-1} - Q_{k}$.

We use the results in Equation \eqref{po_1} and \eqref{po_2} to obtain 
\begin{eqnarray}
    \mathbb{E}(Q^{*}-Q^{\pi_{K}})_{\mu}
    &\leq& \frac{2\gamma}{(1-\gamma)}\left[\sum_{k=1}^{K-1} \gamma^{K-k}\mathbb{E}(|\epsilon_{k}|)_{\mu}\right] + \nonumber\\ && \frac{2R_{\max}\gamma^{K+1}}{(1-\gamma)^{2}} , \label{po_4}
\end{eqnarray}

The first term on the right hand side is called as the algorithmic error, which depends on how good our approximation of the Bellman error is. The second term on the right hand side is called as the statistical error, which is  the error incurred due to the random nature of the system and depends only on the parameters of the MDP as well as the number of iterations of the FQI algorithm. 

The expectation of the error of estimation is taken with respect to any arbitrary distribution on the state action space denoted as $\mu$. The dependence of the expected value of the error of estimation on this distribution is expressed through the constant  $\phi_{\nu,\mu}$, which is a measure of the similarity between the distributions $\nu$ and $\mu$.  %

{\bf Upper Bounding Expectation of Bellman Error: } The upper bound on $\mathbb{E}(Q^{*}-Q^{\pi_{K}})_{\mu}$ in Equation \eqref{po_4} is in terms of  $\mathbb{E}(|\epsilon_{k}|)_{\mu}$, where $\epsilon_{k} =TQ_{k-1} - Q_{k}$ is a measure of how closely our estimate of the $Q$ function at iteration $k$ approximates the function obtained by applying the bellman optimality operator applied to the estimate of the $Q$ function at iteration $k-1$. Intuitively, this error depends on how much data is collected at each iteration, how efficient our solution to the optimization step is to the true solution, and how well our function class can approximate the true Bellman optimality operator applied to the estimate at the end of the previous iteration. Building upon this intuition, we split $\epsilon_{k}$ into four different components as follows.
\begin{eqnarray}
    \epsilon_{k} &=& TQ_{k-1} - Q_{k} \nonumber\\
                 &=& \underbrace{TQ_{k-1}-Q_{k1}}_{\epsilon_{k1}} + \underbrace{Q_{k1} -Q_{k2}}_{\epsilon_{k2}} + \underbrace{Q_{k2} -Q_{k3}}_{\epsilon_{k3}} \nonumber\\
                 & &   +\underbrace{Q_{k3} - Q_{k}}_{\epsilon_{k4}} \nonumber\\
                 &=& \epsilon_{k_1} + \epsilon_{k_2} +\epsilon_{k_3} +\epsilon_{k_4} ,\label{last}
\end{eqnarray}

We use the Lemmas \ref{lem_1}, \ref{lem_2}, \ref{lem_3}, and \ref{lem_4}  to bound the error terms in Equation  \eqref{last}. Before plugging these in Equation \eqref{po_4}, we replace $\epsilon$ in the Lemma results with $\epsilon'=\epsilon\frac{(1-\gamma)^{3}}{6\gamma}$ as these error terms are in a summation in Equation \eqref{po_4}. We also bound the last term on the right hand side of Equation \eqref{po_4} by solving for a value of $K$ that makes the term smaller than  $\frac{\epsilon}{3}$.

\section{Conclusion and Future Work} \label{Conclusion and Future Work}

In this paper, we study a Fitted Q-Iteration with two-layer ReLU neural network parametrization, and find the sample complexity guarantees for the algorithm. Using the convex approach for estimating the Q-function, we show that our approach  achieves a sample complexity of $\tilde{\mathcal{O}}(1/\epsilon^{2})$, which is order-optimal. This demonstrates the first approach for achieving sample complexity beyond linear MDP assumptions for large state space.

This study raises multiple future problems. First is whether we can remove the assumption on the generative model, while estimating Q-function by efficient sampling of past samples. Further, whether the convexity in training that was needed for the results could be extended to more than two layers. Finally, efficient analysis for the error incurred when a sample of cones are chosen rather than the complete set of cones and how to efficiently choose this subset will help with a complete analysis by which Assumption \ref{assump_2} can be relaxed. 

	\bibliography{mybib}
	\onecolumn
	
	\appendix
	\section{Comparison of Result with \cite{xu2020finite}}\label{finite_time_q_learning}

Analysis of $Q$ learning algorithms with neural network function approximation was carried out in \cite{xu2020finite},  where finite time error bounds were studied for estimation of the $Q$ function for MDP's with countable state spaces. The final result in \cite{xu2020finite} is of the form

\begin{eqnarray}
    \frac{1}{T}\sum_{t=1}^{T}\mathbb{E}(Q(s,a,\theta_{t}) - Q^{*}(s,a)) \le \tilde{\mathcal{O}}\left(\frac{1}{\sqrt{T}}\right) + \tilde{\mathcal{O}}(\sqrt{\log(T)})  \label{comp_1}
\end{eqnarray}

Here, $Q(s,a,\theta_{t})$ is the estimate of the $Q$ function obtained using the neural network with the parameters obtained at step $t$ of the $Q$ learning algorithm represented by $\theta_{t}$ and $Q^{*}(s,a)$ is the optimal $Q$ function. $T$ is the total number of iterations of the $Q$ learning algorithm.

The first term on the right hand side of \eqref{comp_1} is monotonically decreasing with respect to $T$. The second term on the right hand side monotonically increases with respect to $T$. This term is present due to the error incurred by the linear approximation of the neural network representing the $Q$ function. Our approach does not require this approximation due to the convex representation of the neural network.  Due to the increasing function of $T$, it is not possible to obtain the number of iteration of the $Q$ learning  required to make the error of estimation smaller than some fixed value. Thus, this result cannot be said to be a sample complexity result.  However, in our results, this is not the case.

	\section{Convex Reformulation with Two-Layer Neural Networks}\label{cones_apdx}
In order to understand the convex reformulation of the squared loss optimization problem, consider the vector $\sigma(Xu_{i})$

\begin{equation}
   \sigma(Xu_{i})= \begin{bmatrix}
        \{\sigma^{'}((x_{1})^{T}u_{i})\} \\
        \{\sigma^{'}((x_{2})^{T}u_{i})\} \\
                     \vdots\\
        \{\sigma^{'}((x_{n})^{T}u_{i})\} 
    \end{bmatrix}
\end{equation}

Now for a fixed  $X \in \mathbb{R}^{n \times d}$, different $u_{i} \in \mathbb{R}^{d \times 1}$ will have different components of $\sigma(Xu_{i})$ that are non zero. For example, if we take the set of all $u_{i}$ such that only the first element of $\sigma(Xu_{i})$ are non zero (i.e, only $(x_{1})^{T}u_{i} \ge 0$ and $(x_{j})^{T}u_{i} < 0$ 
 $\forall j \in [2,\cdots,n]$ ) and denote it by the set $\mathcal{K}_{1}$, then we have 

\begin{equation}
     \sigma(Xu_{i}) =  D_{1}(Xu_{i})   \  \  \  \   \forall u_{i}\in \mathcal{K}_{1}, \nonumber 
\end{equation}

where $D_{1}$ is the  $n \times n$ diagonal matrix with only the first diagonal element equal to $1$ and the rest $0$. Similarly, there exist a set of $u's$ which result in $\sigma(Xu)$ having certain components to be non-zero and the rest zero. For each such combination of zero and non-zero components, we will have a corresponding set of $u_{i}'s$ and a corresponding $n \times n$ Diagonal matrix $D_{i}$.  We define the possible set of such diagonal matrices possible for a given matrix X as 
\begin{eqnarray}
    D_{X} = \{D=diag(\mathbf{1}(Xu \geq 0)):u \in \mathbb{R}^{d}\, ,D \in \mathbb{R}^{n \times n}\},
\end{eqnarray}

where $diag(\mathbf{1}(Xu \geq 0))$ represents a matrix given by

\begin{equation}
    D_{k,j} = \left\{\begin{array}{lr}
    \mathbf{1}(x_{j}^{T}u), & \text{for } k=j\\
    0 & \text{for } k \neq j\end{array}, \right.
\end{equation}

where $\mathbf{1}(x) = 1$ if $x>0$ and $\mathbf{1}(x) = 0$ if $x \le 0.$ Corresponding to each such matrix $D_{i}$, there exists a set of $u_{i}$ given by

\begin{eqnarray}
    \mathcal{K}_{i} =\{ u \in \mathbb{R}^{d}:\sigma(Xu_{i})=D_{i}Xu_{i}, D_{i} \in D_{X} \} \label{cone}
\end{eqnarray}

where $I$ is the $n \times n$ identity matrix. The number of these matrices ${D}_{i}$ is upper bounded by $2^{n}$. From \cite{wang2021hidden} the upper bound is $\mathcal{O}\left(r\left(\frac{n}{r} \right)^{r}\right)$ where $r=rank(X)$. Also, note that the sets $\mathcal{K}_{i}$ form a partition of the space $\mathbb{R}^{d \times 1}$. Using these definitions,  we define the equivalent convex problem to the one in Equation \eqref{ReLU_1} as

\begin{eqnarray}
     \mathcal{L}_{\beta}(v,w) := \argminA_{v,w} (||\sum_{D_{i} \in D_{X}}D_{i}(X(v_{i} - w_{i})) - y||^{2}_{2}  + \label{ReLU_2}\nonumber\\ 
    \beta \sum_{D_{i} \in D_{X}} ||v_{i}||_{2} +||w_{i}||_{2}) 
\end{eqnarray}

where $v=\{v_{i}\}_{i \in 1,\cdots,|D_{X}|}$, $w=\{w_{i}\}_{i \in 1,\cdots,|D_{X}|}$, $v_{i},w_{i} \in \mathcal{K}_{i}$, note that by definition, for any fixed $i \in \{1,\cdots,|D_{X}|\}$ at-least one of $v_{i}$ or $w_{i}$ are zero. If $v^{*},w^{*}$ are the optimal solutions to Equation \eqref{ReLU_2}, the number of neurons $m$ of the original problem in Equation \eqref{ReLU_1} should be greater than the number of elements of $v^{*},w^{*}$, which have at-least one of $v_{i}^{*}$ or $w_{i}^{*}$ non-zero. We denote this value as $m^{*}_{X,y}$, with the subscript $X$ denoting that this quantity depends upon the data matrix $X$ and response $y$. 

We convert $v^{*},w^{*}$ to optimal values of Equation \eqref{ReLU_1}, denoted by $\theta^{*}=(U^{*},\alpha^{*})$, using a function $\psi:\mathbb{R}^{d}\times\mathbb{R}^{d} \rightarrow \mathbb{R}^{d}\times\mathbb{R}$ defined as follows

\begin{eqnarray} 
    \psi(v_{i},w_{i}) &=& \left\{\begin{array}{lr}({v}_{i},1), & \text{if } {w}_{i}=0 \label{ReLU_2_1}\\   
    ({w}_{i},-1), & \text{if }  
     {v}_{i} = 0\\
    (0,0), & \text{if } {v}_{i} = {w}_{i} = 0  \end{array} \right.
    \label{ReLU_4}
\end{eqnarray}

where according to \cite{pmlr-v119-pilanci20a} we have $(u_{i}^{*},\alpha_{i}^{*})=\psi(v_{i}^{*},w_{i}^{*})$, for all $i \in \{1,\cdots,|{D}_{X}|\}$ where $u^{*}_{i},\alpha^{*}_{i}$ are the elements of $\theta^{*}$. Note that restriction of $\alpha_{i}$ to $\{1,-1,0\}$ is shown to be valid in  \cite{pmlr-v162-mishkin22a}. For  $i \in \{|{D}_{X}|+1,\cdots,m\}$ we set $(u_{i}^{*},\alpha_{i}^{*})=(0,0)$.

Since $D_{X}$ is hard to obtain computationally unless $X$ is of low rank, we can construct a subset $\tilde{D} \in D_{X}$ and perform the optimization in Equation \eqref{ReLU_2} by replacing $D_{X}$ with $\tilde{D}$ to get

\begin{eqnarray}
     \mathcal{L}_{\beta}(v,w) := \argminA_{v,w} (||\sum_{D_{i} \in \tilde{D}}D_{i}(X(v_{i} - w_{i})) - y||^{2}_{2}  + \label{ReLU_2_2}\nonumber\\ 
    \beta \sum_{D_{i} \in \tilde{D}} ||v_{i}||_{2} +||w_{i}||_{2}) 
\end{eqnarray}

where $v=\{v_{i}\}_{i \in 1,\cdots,|\tilde{D}|}$, $w=\{w_{i}\}_{i \in 1,\cdots,|\tilde{D}|}$, $v_{i},w_{i} \in \mathcal{K}_{i}$, by definition, for any fixed $i \in \{1,\cdots,|\tilde{D}|\}$ at-least one of $v_{i}$ or $w_{i}$ are zero.

 The required condition for $\tilde{D}$ to be a sufficient replacement for $D_{X}$ is as follows. Suppose $(v,w)=(\bar{v}_{i},\bar{w}_{i})_{i \in (1,\cdots,|\tilde{D}|)}$ denote the optimal solutions of Equation \eqref{ReLU_2_2}. Then we require

\begin{eqnarray}
    m \ge \sum_{D_{i} \in \tilde{D}} |\{ \bar{v}_{i}: \bar{v}_{i} \neq 0 \} \cup  \{ \bar{w}_{i}: \bar{w}_{i} \neq 0 \}| \label{ReLU_2_3}
\end{eqnarray}

Or,  the number of neurons in the neural network are greater than the number of indices $i$ for which at-least one of $v_{i}^{*}$ or $w_{i}^{*}$ is non-zero. Further,

\begin{eqnarray}
  diag(Xu_{i}^{*} \geq 0: i \in [m]) \in \tilde{D}   \label{ReLU_2_3_1}
\end{eqnarray}

In other words,  the diagonal matrices induced by the optimal $u_{i}^{*}$'s of Equation \eqref{ReLU_1} must be included in our sample of diagonal matrices. 
This is proved in Theorem 2.1 of \cite{pmlr-v162-mishkin22a}.

A computationally efficient method for obtaining $\tilde{D}$ and obtaining the optimal values of the Equation \eqref{ReLU_1}, is laid out in \cite{pmlr-v162-mishkin22a}. In this method we first get our sample of diagonal matrices $\tilde{D}$ by first sampling a fixed number of vectors from a $d$ dimensional standard multivariate distribution, multiplying the vectors with the data matrix $X$ and then forming the diagonal matrices based of which co-ordinates are positive. Then  we solve an optimization similar to the one in Equation \eqref{ReLU_2}, without the constraints, that its parameters belong to sets of the form $\mathcal{K}_{i}$ as follows.

\begin{eqnarray}
     \mathcal{L}^{'}_{\beta}(p) := \argminA_{p} (||\sum_{D_{i} \in \tilde{D}}D_{i}(Xp_{i}) - y||^{2}_{2}  + \beta \sum_{D_{i} \in \tilde{D}} ||p_{i}||_{2}) ,\label{ReLU_3}
\end{eqnarray}

where $p \in \mathbb{R}^{d \times |\tilde{D}|}$ . In order to satisfy the constraints of the form given in Equation \eqref{ReLU_2}, this step is followed by a cone decomposition step. This is implemented through a function $\{\psi_{i}^{'}\}_{i \in \{1,\cdots,|\tilde{D}|\}}$.  Let $p^{*}=\{p^{*}_{i}\}_{i \in \{1,\cdots,|\tilde{D}|\}}$ be the optimal solution of Equation \eqref{ReLU_3}. For each $i$ we define a function $\psi_{i}^{'}:\mathbb{R}^{d} \rightarrow \mathbb{R}^{d}\times\mathbb{R}^{d}$ as 

\begin{eqnarray}
    \psi_{i}^{'}(p_{i}) &=& (v_{i},w_{i}) \label{ReLU_3_1}\\
     \textit{such that } p&=& {v}_{i} - {w}_{i}, \textit{and }   {v}_{i},{w}_{i} \in \mathcal{K}_{i} \nonumber
\end{eqnarray}

Then we obtain $\psi(p^{*}_{i})=(\bar{v}_{i},\bar{w}_{i})$. As before, at-least one of $v_{i}$, $w_{i}$ is $0$. Note that in practice we do not know if the conditions in Equation \eqref{ReLU_2_3} and \eqref{ReLU_2_3_1} are satisfied for a given sampled $\tilde{D}$. We express this as follows. If $\tilde{D}$  was the full set of Diagonal matrices  then we would have $(\bar{v}_{i},\bar{w}_{i})={v}^{*}_{i},{w}^{*}_{i}$ and $\psi(\bar{v}_{i},\bar{w}_{i})=(u_{i}^{*},\alpha_{i}^{*})$ for all $i \in (1,\cdots,|D_{X}|)$. However, since that is not the case and $\tilde{D} \in D_{X}$, this means that $\{\psi(\bar{v}_{i},\bar{w}_{i})\}_{i \in (1,\cdots,|\tilde{D}|)}$ is an optimal solution of a non-convex optimization different from the one in Equation \eqref{ReLU_1}. We denote this non-convex optimization as $\mathcal{L}_{|\tilde{D}|}(\theta)$ defined as 

\begin{equation}
    \mathcal{L}_{|\tilde{D}|}(\theta) = \argminA_{\theta} ||\sum_{i=1}^{m^{'}}\sigma(Xu_{i})\alpha_{i}- y||_{2}^{2} + \beta  \sum_{i=1}^{m}||u_{i}||_{2}^{2} + \beta \sum_{i=1}^{m}|\alpha_{i}|^{2} \label{ReLU_6}, 
\end{equation}

where $m^{'} = |\tilde{D}|$  or the size of the sampled diagonal matrix set. In order to quantify the error incurred  due to taking a subset of $D_{X}$, we assume that the expectation of the absolute value of the difference between the neural networks corresponding to the optimal solutions of the non-convex optimizations given in Equations \eqref{ReLU_6} and \eqref{ReLU_1}  is upper bounded by a constant depending on the size of $\tilde{D}$. The formal assumption and its justification is given in Assumption \ref{assump_2}.

\section{Supplementary lemmas and Definitions}\label{sup_lem}

Here we provide some definitions and results that will be used to prove the lemmas stated in the  paper.

\begin{definition} \label{def_8}
   For a given set $ Z \in \mathbb{R}^{n}$, we define the Rademacher complexity of the set $Z$ as 
   \begin{equation}
   Rad(Z) = \mathbb{E} \left(\sup_{z \in Z} \frac{1}{n} \sum_{i=1}^{d}\Omega_{i}z_{i}\right)    
   \end{equation}
   where $\Omega_{i}$ is random variable such that $P(\Omega_{i}=1)=\frac{1}{2}$,  $P(\Omega_{i}=-1)=\frac{1}{2}$ and $z_{i}$ are the co-ordinates of $z$ which is an element of the set $Z$
\end{definition}

\begin{lemma} \label{sup_lem_0}
Consider a set of observed data denoted by $ z = \{z_{1},z_{2},\cdots\,z_{n}\} \in \mathbb{R}^{n}$, a parameter space  $\Theta$, a loss function $\{l:\mathbb{R} \times \Theta \rightarrow \mathbb{R}\}$ where  $0 \le l(\theta,z) \le 1$  $\forall (\theta,z) \in \Theta \times \mathbb{R}$. The empirical risk for a set of observed data as $R(\theta)=\frac{1}{n} \sum_{i=1}^{n}l(\theta,z_{i})$ and the population risk  as $r(\theta)= \mathbb{E}l(\theta,\tilde{z_{i}})$, where $\tilde{z_{i}}$ is a co-ordinate of $\tilde{z}$ sampled from some distribution over $Z$.

We define a set of functions denoted by $\mathcal{L}$ as 

\begin{equation}
    \mathcal{L}=\{z \in Z \rightarrow l(\theta,z) \in \mathbb{R}:\theta \in \Theta \}
\end{equation}

Given $z=\{z_{1},z_{2},z_{3}\cdots,z_{n}\}$ we further define a set $\mathcal{L} \circ z$ as 

\begin{equation}
    \mathcal{L} \circ z \ =\{ (l(\theta,z_{1}),l(\theta,z_{2}),\cdots,l(\theta,z_{n})) \in \mathbb{R}^{n} : \theta \in \Theta\}
\end{equation}

Then, we have 

\begin{equation}
    \mathbb{E}\sup_{\theta \in \Theta} |\{r(\theta)-R(\theta)\}| \le 2\mathbb{E} \left(Rad(\mathcal{L} \circ z)\right)
\end{equation}

If the data is of the form $z_{i}=(x_{i},y_{i}), x \in X, y \in Y$ and the loss function is of the form $l(a_{\theta}(x),y)$, is $L$ lipschitz and $a_{\theta}:\Theta{\times}X \rightarrow \mathbb{R}$, then we have 

\begin{equation}
    \mathbb{E}\sup_{\theta \in \Theta} |\{r(\theta)-R(\theta)\}| \le 2{L}\mathbb{E} \left(Rad(\mathcal{A} \circ \{x_{1},x_{2},x_{3},\cdots,x_{n}\})\right)
\end{equation}

where \begin{equation}
    \mathcal{A} \circ \{x_{1},x_{2},\cdots,x_{n}\}\ =\{ (a(\theta,x_{1}),a(\theta,x_{2}),\cdots,a(\theta,x_{n})) \in \mathbb{R}^{n} : \theta \in \Theta\}
\end{equation}

\end{lemma}

The detailed proof of the above statement is given in (Rebeschini, P. (2022). Algorithmic Foundations of Learning [Lecture Notes]. https://www.stats.ox.ac.uk/~rebeschi/teaching/AFoL/20/material/. The upper bound for $ \mathbb{E}\sup_{\theta \in \Theta} (\{r(\theta)-R(\theta)\})$ is proved in the aformentioned reference. However, without loss of generality the same proof holds for the upper bound for $ \mathbb{E}\sup_{\theta \in \Theta} (\{R(\theta)-r(\theta)\})$. Hence the upper bound for $ \mathbb{E}\sup_{\theta \in \Theta}|\{r(\theta)-R(\theta)\}|$ can be established.

\begin{lemma} \label{sup_lem_1}
Consider two random random variable $x \in \mathcal{X} $ and  $y,y^{'} \in \mathcal{Y}$. Let $\mathbb{E}_{x,y}, \mathbb{E}_{x}$ and $\mathbb{E}_{y|x}$, $\mathbb{E}_{y^{'}|x}$  denote the expectation with respect to the joint distribution of $(x,y)$, the marginal distribution of $x$, the conditional distribution of $y$ given $x$ and the conditional distribution of $y^{'}$ given $x$ respectively . Let $f_{\theta}(x)$ denote a bounded measurable function of $x$ parameterised by some parameter $\theta$ and $g(x,y)$ be bounded measurable function of both $x$ and $y$.

Then we have

\begin{equation}
    \argminA_{f_{\theta}}\mathbb{E}_{x,y}\left(f_{\theta}(x)-g(x,y)\right)^{2}=\argminA_{f_{\theta}} \left(\mathbb{E}_{x,y}\left(f_{\theta}(x)-\mathbb{E}_{y^{'}|x}(g(x,y^{'})|x)\right)^{2}\right) \label{sup_lem_1_1}
\end{equation}    
\end{lemma}

\begin{proof}
Denote the left hand side of Equation \eqref{sup_lem_1_1} as $\mathbb{X}_{\theta}$, then add and subtract $\mathbb{E}_{y|x}(g(x,y)|x)$ to it to get 
\begin{eqnarray}
     \mathbb{X}_{\theta}&=& \argminA_{f_{\theta}}\left(\mathbb{E}_{x,y}\left(f_{\theta}(x)-\mathbb{E}_{y^{'}|x}(g(x,y^{'})|x)+\mathbb{E}_{y^{'}|x}(g(x,y^{'})|x)-g(x,y)\right)^{2}\right) \label{sup_lem_1_2}\\\
     &=&  \argminA_{f_{\theta}}\Big(\mathbb{E}_{x,y}\left(f_{\theta}(x)-\mathbb{E}_{y^{'}|x}(g(x,y^{'})|x)\right)^{2} + \mathbb{E}_{x,y}\left(y-\mathbb{E}_{y^{'}|x}(g(x,y^{'})|x)\right)^{2} -2\mathbb{E}_{x,y}\Big(f_{\theta}(x)-\nonumber\\ 
     &&\mathbb{E}_{y^{'}|x}(g(x,y^{'})|x)\Big)\left(g(x,y)-\mathbb{E}_{y^{'}|x}(g(x,y^{'})|x)\right)\Big) \label{sup_lem_1_3}
\end{eqnarray}

Consider the third term on the right hand side of Equation \eqref{sup_lem_1_3}
\begin{eqnarray}
    2\mathbb{E}_{x,y}\left(f_{\theta}(x)-\mathbb{E}_{y^{'}|x}(g(x,y^{'})|x)\right)\left(g(x,y)-  \mathbb{E}_{y^{'}|x}(g(x,y^{'})|x)\right) &=& 2\mathbb{E}_{x}\mathbb{E}_{y|x}\left(f_{\theta}(x)-\mathbb{E}_{y^{'}|x} (g(x,y^{'})|x)\right)\nonumber\\
    &&\left(g(x,y)-\mathbb{E}_{y^{'}|x}(g(x,y^{'})|x)\right)\nonumber\\
    &&\label{sup_lem_1_4}\\
    &=& 2\mathbb{E}_{x}\left(f_{\theta}(x)-\mathbb{E}_{y^{'}|x}(g(x,y^{'})|x)\right)\nonumber\\
    &&  \mathbb{E}_{y|x}\left(g(x,y)-\mathbb{E}_{y^{'}|x}(g(x,y^{'})|x)\right) \label{sup_lem_1_5}\\
    &=& 2\mathbb{E}_{x}\left(f_{\theta}(x)-\mathbb{E}_{y^{'}|x}(g(x,y^{'})|x)\right)\nonumber\\
    &&  \left(\mathbb{E}_{y|x}(g(x,y))-\mathbb{E}_{y|x}\left(\mathbb{E}_{y^{'}|x}(g(x,y^{'})|x)\right)\right) \label{sup_lem_1_6}\\
    &=& 2\mathbb{E}_{x}\left(f_{\theta}(x)-\mathbb{E}(y|x)\right)\left(\mathbb{E}_{y|x}(g(x,y))-\mathbb{E}_{y^{'}|x}(g(x,y^{'})|x)\right) \nonumber\\
    &&  \label{sup_lem_1_7}\\
    &=& 0
\end{eqnarray}

Equation \eqref{sup_lem_1_4} is obtained by writing $\mathbb{E}_{x,y}=\mathbb{E}_{x}\mathbb{E}_{y|x}$ from the law of total expectation. Equation \eqref{sup_lem_1_5} is obtained from  \eqref{sup_lem_1_4} as the term $f_{\theta}(x)-\mathbb{E}_{y^{'}|x}(g(x,y^{'})|x)$ is not a function of $y$. Equation \eqref{sup_lem_1_6} is obtained from \eqref{sup_lem_1_5} as $\mathbb{E}_{y|x}\left(\mathbb{E}_{y^{'}|x}(g(x,y^{'})|x)\right)=\mathbb{E}_{y^{'}|x}(g(x,y^{'})|x)$ because $\mathbb{E}_{y^{'}|x}(g(x,y^{'})|x)$ is not a function of $y$ hence is constant with respect to the expectation operator $\mathbb{E}_{y|x}$. 

Thus plugging in value of $  2\mathbb{E}_{x,y}\left(f_{\theta}(x)-\mathbb{E}_{y^{'}|x}(g(x,y^{'})|x)\right)\left(g(x,y)-  \mathbb{E}_{y^{'}|x}(g(x,y^{'})|x)\right)$ in Equation \eqref{sup_lem_1_3} we get 

\begin{equation}
    \argminA_{f_{\theta}}\mathbb{E}_{x,y}\left(f_{\theta}(x)-g(x,y)\right)^{2} =  \argminA_{f_{\theta}} \left(\mathbb{E}_{x,y}\left(f_{\theta}(x)-\mathbb{E}_{x,y^{'}}(g(x,y^{'})|x)\right)^{2} + \mathbb{E}_{x,y}\left(g(x,y)-\mathbb{E}_{y^{'}|x}(g(x,y^{'})|x)\right)^{2}\right) \label{sup_lem_1_8}
\end{equation}

Note that the second term on the right hand side of Equation \eqref{sup_lem_1_8} des not depend on $f_{\theta}(x)$ therefore we can write Equation \eqref{sup_lem_1_8} as 

\begin{equation}
    \argminA_{f_{\theta}}\mathbb{E}_{x,y}\left(f_{\theta}(x)-g(x,y)\right)^{2} =  \argminA_{f_{\theta}} \left(\mathbb{E}_{x,y}\left(f_{\theta}(x)-\mathbb{E}_{y^{'}|x}(g(x,y^{'})|x)\right)^{2}\right) \label{sup_lem_1_9}
\end{equation}

Since the right hand side of Equation \eqref{sup_lem_1_9} is not a function of $y$ we can replace $\mathbb{E}_{x,y}$ with $\mathbb{E}_{x}$ to get 

\begin{equation}
    \argminA_{f_{\theta}}\mathbb{E}_{x,y}\left(f_{\theta}(x)-g(x,y)\right)^{2} =  \argminA_{f_{\theta}} \left(\mathbb{E}_{x}\left(f_{\theta}(x)-\mathbb{E}_{y^{'}|x}(g(x,y^{'})|x)\right)^{2}\right) \label{sup_lem_1_10}
\end{equation}
\end{proof}

\begin{lemma} \label{sup_lem_3}
Consider an optimization of the form given in Equation \eqref{ReLU_2_2} with the regularization term $\beta = 0$ denoted by $\mathcal{L}_{|\tilde{D}|}$ and it's convex equivalent denoted by $\mathcal{L}_{0}$. Then the value of these two loss functions evaluated at $(v,w)=(v_{i},w_{i})_{i \in \{1,\cdots,|\tilde{D}|\}}$ and $\theta=\psi(v_{i},w_{i})_{i \in \{1,\cdots,|\tilde{D}|\}}$ respectively are equal and thus we have

\begin{equation}
\mathcal{L}_{|\tilde{D}|}(\psi(v_{i},w_{i})_{i \in \{1,\cdots,|\tilde{D}|\}}) = \mathcal{L}_{0}((v_{i},w_{i})_{i \in \{1,\cdots,|\tilde{D}|\}})   
\end{equation}

\end{lemma}

\begin{proof}
    Consider the loss functions in Equations  \eqref{ReLU_2}, \eqref{ReLU_3} with $\beta=0$ are as follows

\begin{eqnarray}
    \mathcal{L}_{0}((v_{i},w_{i})_{i \in \{1,\cdots,|\tilde{D}|\}}) &=& ||\sum_{D_{i} \in \tilde{D}}D_{i}(X(v_{i} - w_{i})) - y||^{2}_{2} \label{sup_lem_3_1}\\ 
    \mathcal{L}_{|\tilde{D}|}(\psi(v_{i},w_{i})_{i \in \{1,\cdots,|\tilde{D}|\}}) &=& ||\sum_{i=1}^{|\tilde{D}|}\sigma(X\psi(v_{i},w_{i})_{1})\psi(v_{i},w_{i})_{2}- y||_{2}^{2},  \label{sup_lem_3_2}
\end{eqnarray}

where $\psi(v_{i},w_{i})_{1}$, $\psi(v_{i},w_{i})_{2}$ represent the first and second coordinates of $\psi(v_{i},w_{i})$ respectively.

For any fixed $i \in \{1,\cdots,|\tilde{D}|\}$ consider the two terms 

\begin{eqnarray}
D_{i}(X(v_{i}-w_{i})) \label{sup_lem_3_3}\\
\sigma(X\psi(v_{i},w_{i})_{1})\psi(v_{i},w_{i})_{2}   \label{sup_lem_3_4}
\end{eqnarray}

For a fixed $i$ either $v_{i}$ or $w_{i}$ is zero. In case both are zero, both of the terms in Equations \eqref{sup_lem_3_3} and  \eqref{sup_lem_3_4} are zero as $\psi(0,0)=(0,0)$. Assume that for a given $i$ $w_{i}=0$. Then we have $\psi(v_{i},w_{i})=(v_{i},1)$. Then equations \eqref{sup_lem_3_3}, \eqref{sup_lem_3_4}  are.

\begin{eqnarray}
D_{i}(X(v_{i})  \label{sup_lem_3_5}\\
\sigma(X(v_{i}))    \label{sup_lem_3_6}
\end{eqnarray}

But by definition of $v_{i}$ we have $D_{i}(X(v_{i})=\sigma(X(v_{i}))$, therefore Equations \eqref{sup_lem_3_5}, \eqref{sup_lem_3_6} are equal. Alternatively if for a given $i$ $v_{i}=0$, then $\psi(v_{i},w_{i})=(w_{i},-1)$, then the terms in \eqref{sup_lem_3_3}, \eqref{sup_lem_3_4}  become.

\begin{eqnarray}
-D_{i}(X(w_{i})  \label{sup_lem_3_7}\\
-\sigma(X(w_{i}))    \label{sup_lem_3_8}
\end{eqnarray}

By definition of $w_{i}$ we have $D_{i}(X(w_{i})=\sigma(X(w_{i}))$, then the terms 
in \eqref{sup_lem_3_7}, \eqref{sup_lem_3_7} are equal. 
Since this is true for all $i$, we have 

\begin{equation}
\mathcal{L}_{|\tilde{D}|}(\psi(v_{i},w_{i})_{i \in \{1,\cdots,|\tilde{D}|\}}) = \mathcal{L}_{0}((v_{i},w_{i})_{i \in \{1,\cdots,|\tilde{D}|\}}) 
 \label{sup_lem_3_9}    
\end{equation}

\end{proof}

\begin{lemma} \label{sup_lem_5}
The function  $Q_{\theta}(x)$ defined in equation \eqref{ReLU_1_2} is Lipschitz continuous in $\theta$, where $\theta$ is considered a vector in $\mathbb{R}^{(d+1)m}$ with the assumption that the set of all possible $\theta$ belong to the set  $\mathcal{B} = \{ \theta:  |\theta^{*}-\theta|_{1} < 1\}$, where $\theta^{*}$ is some fixed value.
\end{lemma}
\begin{proof}

First we show that for all $\theta_{1} = \{u_{i},\alpha_{i}\}, \theta_{2}= \{u^{'}_{i},\alpha^{'}_{i}\} \in \mathcal{B}$  we have $\alpha_{i}=\alpha^{'}_{i}$ for all $i \in (1,\cdots,m)$

Note that 

\begin{equation}
    |\theta_{1} - \theta_{2}|_{1} = \sum_{i=1}^{m}|u_{i}-u^{'}_{i}|_{1}  +  \sum_{i=1}^{m}|\alpha_{i}-\alpha^{'}_{i}|,
\end{equation}

where $|u_{i}-u^{'}_{i}|_{1} = \sum_{j=1}^{d}|u_{i_{j}}-u^{'}_{i_{j}}|$ with $u_{i_{j}}, u^{'}_{i_{j}}$ denote the $j^{th}$ component of $u_{i}, u^{'}_{i}$ respectively.

By construction $\alpha_{i}, \alpha^{'}_{i}$ can only be $1$, $-1$ or $0$. Therefore if $\alpha_{i}\neq\alpha^{'}_{i}$ then $|\alpha_{i}-\alpha^{'}_{i}|=2$ if both non zero or $|\alpha_{i}-\alpha^{'}_{i}|=1$ if one is zero. Therefore $|\theta_{1} - \theta_{2}|_{1} \geq 1$. Which leads to a contradiction.  

Therefore  $\alpha_{i}=\alpha^{'}_{i}$ for all $i$ and we also have 

\begin{equation}
    |\theta_{1} - \theta_{2}|_{1} = \sum_{i=1}^{m}|u_{i}-u^{'}_{i}|_{1}
\end{equation}

$Q_{\theta}(x)$ is defined as 

\begin{equation}
Q_{\theta}(x)=\sum_{i=1}^{m}\sigma^{'}(x^{T}u_{i})\alpha_{i} \label{sup_lem_5_1}
\end{equation}

From Proposition $1$ in \cite{10.5555/3327144.3327299} the function $Q_{\theta}(x)$ is Lipschitz continuous in $x$, therefore there exist $l > 0$ such that 

\begin{eqnarray}
|Q_{\theta}(x)- Q_{\theta}(y)|  &\le&    l|x-y|_{1} \label{sup_lem_5_2} \\
|\sum_{i=1}^{m}\sigma^{'}(x^{T}u_{i})\alpha_{i} - \sum_{i=1}^{m}\sigma^{'}(y^{T}u_{i})\alpha_{i}|  &\le& l|x-y|_{1}  \label{sup_lem_5_3}
\end{eqnarray}

If we consider a single neuron of $Q_{\theta}$, for example $i=1$, we have  $l_{1} > 0$ such that 
\begin{eqnarray}
|\sigma^{'}(x^{T}u_{1})\alpha_{i} - \sigma^{'}(y^{T}u_{1})\alpha_{i}|  &\le&  l_{1}|x-y|_{1}  \label{sup_lem_5_4}
\end{eqnarray}

Now consider Equation \eqref{sup_lem_5_4}, but instead of considering the left hand side a a function of $x,y$ consider it a function of $u$ where we consider the difference between $\sigma^{'}(x^{T}u)\alpha_{i}$ evaluated at $u_{1}$ and $u^{'}_{1}$ such that 
\begin{eqnarray}
|\sigma^{'}(x^{T}u_{1})\alpha_{i} - \sigma^{'}(x^{T}u^{'}_{1})\alpha_{i}|  &\le&  l^{x}_{1}|u_{1}-u^{'}_{1}|_{1}  \label{sup_lem_5_5}
\end{eqnarray}

for some $l^{x}_{1} > 0$.

Similarly, for all other $i$ if we change $u_{i}$ to $u^{'}_{i}$ to be unchanged we have 

\begin{eqnarray}
|\sigma^{'}(x^{T}u_{i})\alpha_{i} - \sigma^{'}(x^{T}u^{'}_{i})\alpha_{i}|  &\le&  l^{x}_{i}|u_{i}-u^{'}_{i}|_{1}  \label{sup_lem_5_6}
\end{eqnarray}

for all $x$ if both $\theta_{1}, \theta_{2} \in \mathcal{B}$.

Therefore we obtain

\begin{eqnarray}
|\sum_{i=1}^{m}\sigma^{'}(x^{T}u_{i})\alpha_{i} - \sum_{i=1}^{m}\sigma^{'}(x^{T}u^{'}_{i})\alpha_{i}|  &\le&  \sum_{i=1}^{m}|\sigma^{'}(x^{T}u_{i})\alpha_{i} -(x^{T}u^{'}_{i})\alpha_{i}|   \label{sup_lem_5_7}\\
                                                                                                       &\le&  \sum_{i=1}^{m}l^{x}_{i}|u_{i}-u^{'}_{i}|_{1}  \label{sup_lem_5_8}\\
                                                                                                       &\le&  (\sup_{i}l_{i}^{x})\sum_{i=1}^{m}|u_{i}-u^{'}_{i}|_{1}  \label{sup_lem_5_9}\\
                                                                                                       &\le&  (\sup_{i}l_{i}^{x})|\theta_{1}-\theta_{2}|  \label{sup_lem_5_10}
\end{eqnarray}

This result for a fixed $x$. If we take the supremum over $x$ on both sides we get

\begin{eqnarray}
\sup_{x}|\sum_{i=1}^{m}\sigma^{'}(x^{T}u_{i})\alpha_{i} - \sum_{i=1}^{m}\sigma^{'}(x^{T}u^{'}_{i})\alpha_{i}| 
                                                                                                       &\le&  (\sup_{i,x}l_{i}^{x})|\theta_{1}-\theta_{2}|  \label{sup_lem_5_11}
\end{eqnarray}

Denoting $(\sup_{i,x}l_{i}^{x})=l$, we get 
\begin{eqnarray}
|\sum_{i=1}^{m}\sigma^{'}(x^{T}u_{i})\alpha_{i} - \sum_{i=1}^{m}\sigma^{'}(x^{T}u^{'}_{i})\alpha_{i}| 
                                                                                                       &\le&  l|\theta_{1}-\theta_{2}|_{1}  \label{sup_lem_5_12}\\
                                                                                                       && \forall x \in  \mathbb{R}^{d}
\end{eqnarray}
\end{proof}

	\section{Proof of Theorem \ref{thm}} \label{thm proof} 
\begin{proof}

    For ease of notations, let $Q_{1}, Q_{2}$ be two real valued functions on the state action space. The expression $Q_{1} \ge Q_{2}$ implies $Q_{1}(s,a) \ge Q_{2}(s,a)$ $\forall (s,a) \in \mathcal{S}\times\mathcal{A}$. 

    $Q_{k}$ denotes our estimate of the action value function at step $k$ of Algorithm \ref{algo_1} and  $Q^{\pi_{k}}$ denotes the action value function induced by the policy $\pi_{k}$ which is the greedy policy with respect to $Q_{k}$.

    Consider $\epsilon_{k+1}= TQ_{k}-Q_{k+1}$.

    \begin{eqnarray}
        TQ_{k} \geq T^{\pi^{*}} Q_{k} \label{thm_1_1}
    \end{eqnarray}

This follows from the definition of $T^{\pi^{*}}$ and $T$ in Equation \eqref{ps_3} and \eqref{ps_4}, respectively.

Thus we get,
    \begin{eqnarray}
        Q^{*}-Q_{k+1} &=& T^{\pi^{*}}Q^{*} -  Q_{k+1}  \label{thm_1_1_1}\\
                      &=& T^{\pi^{*}}Q^{*} -  T^{\pi^{*}}Q_{k} + T^{\pi^{*}}Q_{k} - TQ_{k} + TQ_{k} - Q_{k+1}  \label{thm_1_1_2}\\
                      &=& r(s,a) + \gamma P^{\pi^{*}}Q^{*} - (r(s,a) +  \gamma P^{\pi^{*}}Q_{k})   + (r(s,a) +  \gamma P^{\pi^{*}}Q_{k})  -  (r(s,a) +  \gamma P^{*}Q_{k}) + \epsilon_{k+1}    \nonumber \label{thm_1_1_3}\\  
                      &=& \gamma P^{\pi^{*}}(Q^{*}-Q_{k}) + \gamma P^{\pi^{*}}Q_{k} - \gamma P^{*}Q_{k} + \epsilon_{k+1} \label{thm_1_1_4}\\  
                      &\le& \gamma(P^{\pi^{*}}(Q^{*}-Q_{k})) + \epsilon_{k+1} \label{thm_1_1_5}
    \end{eqnarray}

Right hand side of Equation \eqref{thm_1_1_1} is obtained by writing $Q^{*} = T^{\pi^{*}}Q^{*}$. This is because the function $Q^{*}$ is a stationary point with respect to the operator $T^{\pi^{*}}$. Equation \eqref{thm_1_1_2} is obtained from \eqref{thm_1_1_1} by adding and subtracting $T^{\pi^{*}}Q_{k}$. Equation \eqref{thm_1_1_5} is obtained from \eqref{thm_1_1_4} as $P^{\pi^{*}}Q_{k} \le P^{*}Q_{k}$ and $P^{*}$ is the operator with respect to the greedy policy of $Q_{k}$.

By recursion on $k$, we get,

\begin{equation}
    Q^{*}-Q_{K} \leq \sum_{k=0}^{K-1} \gamma^{K-k-1} (P^{\pi^{*}})^{K-k-1}\epsilon_{k} +  \gamma^{K} (P^{\pi^{*}})^{K}(Q^{*}-Q_{0}) \label{thm_1_2}
\end{equation}

using $TQ_{K} \ge T^{\pi^{*}}Q_{K}$ (from definition of $T^{\pi^{*}}$) and $TQ_{K}=T^{\pi_{K}}Q_{K}$ as $\pi_{k}$ is the greedy policy with respect to $Q_{k}$ hence $T^{\pi_{K}}$ acts on it the same way $T$ does.

Similarly we write,

\begin{eqnarray}
    Q^{*}-Q^{\pi_{K}}  &=& T^{\pi^{*}}Q^{*} - T^{\pi_{K}}Q^{\pi_{K}} \label{thm_1_2_1}\\
     &=&  T^{\pi^{*}}Q^{*} -  T^{\pi^{*}}Q_{K} + T^{\pi^{*}}Q_{K} - TQ_{K} + TQ_{k} - T^{\pi_{K}}Q^{\pi_{K}} \label{thm_1_2_2}\\
     &\le&  T^{\pi^{*}}Q^{*} -  T^{\pi^{*}}Q_{K} + TQ_{k} - T^{\pi_{K}}Q^{\pi_{K}} \label{thm_1_2_3}\\
     &\le&  r(s,a) + \gamma P^{\pi^{*}}Q^{*} - (r(s,a) + \gamma P^{\pi^{*}}Q_{K}) +  (r(s,a) + \gamma P^{*}Q_{k}) - (r(s,a)+ \gamma P^{\pi_{K}}Q^{\pi_{K}}) \label{thm_1_2_4}\\
     &=& \gamma P^{\pi^{*}}(Q^{*}-Q_{K}) + \gamma P^{\pi_{K}}(Q_{K}-Q^{\pi_{K}}) \label{thm_1_2_5}\\
     &=& \gamma P^{\pi^{*}}(Q^{*}-Q_{K}) + \gamma P^{\pi_{K}}(Q_{K}-Q^{*}+Q^{*}-Q^{\pi_{K}}) \label{thm_1_2_6}
\end{eqnarray}

The right hand side of Equation \eqref{thm_1_2_2} is obtained by adding and subtracting $T^{\pi^{*}}Q_{K}$ and $TQ_{K}$ to the right hand side of Equation \eqref{thm_1_2_1}. Equation \eqref{thm_1_2_3} is obtained from \eqref{thm_1_2_2} by noting that the term $T^{\pi^{*}}Q_{K} - TQ_{K}$ is non-positive for all $(s,a)$ as $T$ is the greedy operator on a Q function and results in a higher or equal value than any other operator. Equation \eqref{thm_1_2_5} is obtained from \eqref{thm_1_2_4} by writing $P^{\pi_{k}}Q^{\pi_{k}}=P^{*}Q^{\pi_{k}}$. This is true as $\pi_{k}$ is the greedy policy with respect to $Q^{\pi_{k}}$, hence the operator $P^{*}$ acts on $Q^{\pi_{k}}$ in the same way as $P^{\pi_{k}}$. Equation \eqref{thm_1_2_6} is obtained from \eqref{thm_1_2_5} by adding and subtracting $Q^{*}$ to the second term on the right hand side.

By rearranging the terms in Equation \eqref{thm_1_2_6} we get 
\begin{eqnarray}
    Q^{*}-Q^{\pi_{K}} &\leq&  \gamma P^{\pi^{*}}(Q^{*}-Q_{K}) + \gamma P^{\pi_{K}}(Q_{K}-Q^{*}+Q^{*}-Q^{\pi_{K}}) \label{thm_1_2_7}\\
                      &\leq&  \gamma (P^{\pi^{*}} - P^{\pi_{K}})(Q^{*}-Q_{K}) + \gamma  P^{\pi_{K}}(+Q^{*}-Q^{\pi_{K}}) \label{thm_1_2_8}
 \end{eqnarray}

Thus, we have 
 \begin{eqnarray}
     (I-\gamma P^{\pi_{K}})(Q^{*}-Q^{\pi_{K}})  &\leq& \gamma (P^{\pi^{*}} - P^{\pi_{K}})(Q^{*}-Q_{K}) \label{thm_1_2_9}
\end{eqnarray}

This implies
\begin{eqnarray}
     (Q^{*}-Q^{\pi_{K}})  &\leq& \gamma(I-\gamma P^{\pi_{K}})^{-1}\left(P^{\pi^{*}} - P^{\pi_{K}}\right)(Q^{*}-Q_{K}) \label{thm_1_2_10}
\end{eqnarray}

Equation \eqref{thm_1_2_10} is obtained from \eqref{thm_1_2_9} by taking the inverse of the operator $(I-\gamma P^{\pi_{K}})$  on both sides of \eqref{thm_1_2_9}

Plugging in value of $(Q^{*}-Q_{K})$ from Equation \eqref{thm_1_2}  in Equation \eqref{thm_1_2_10},  we obtain 

\begin{eqnarray}
    Q^{*}-Q^{\pi_{K}} 
    &\leq& \gamma(I-\gamma P^{\pi_{K}})^{-1} \left(P^{\pi^{*}} - P^{\pi_{K}}\right)\left[\sum_{k=1}^{K-1} \gamma^{K-k-1} (P^{\pi^{*}})^{K-k-1}\epsilon_{k} +  \gamma^{K} (P^{\pi^{*}})^{K}(Q^{*}-Q_{0})\right]\label{thm_1_2_11}\\
    &\leq& \gamma(I-\gamma P^{\pi_{K}})^{-1} \left(P^{\pi^{*}} - P^{\pi_{K}}\right)\left[\sum_{k=1}^{K-1} \gamma^{K-k-1} (P^{\pi^{*}})^{K-k-1}|\epsilon_{k}| +  \gamma^{K} (P^{\pi^{*}})^{K}(Q^{*}-Q_{0})\right]\label{thm_1_2_11_1}\\
    &\leq& \gamma(I-\gamma P^{\pi_{K}})^{-1} \left(P^{\pi^{*}} + P^{\pi_{K}}\right)\left[\sum_{k=1}^{K-1} \gamma^{K-k-1}(P^{\pi^{*}})^{K-k-1}|\epsilon_{k}| +  \gamma^{K} (P^{\pi^{*}})^{K}(Q^{*}-Q_{0}) \right]\label{thm_1_2_11_2}\\
    &\leq& \gamma(I-\gamma P^{\pi_{K}})^{-1} \left(P^{\pi^{*}} + P^{\pi_{K}}\right)\left[\sum_{k=1}^{K-1} \gamma^{K-k-1}(P^{\pi^{*}})^{K-k-1}|\epsilon_{k}| +  \gamma^{K} (P^{\pi^{*}})^{K}(Q^{*}-Q_{0}) \right]\label{thm_1_2_12}\\
    &\leq& \gamma(I-\gamma P^{\pi_{K}})^{-1} \Bigg[\sum_{k=1}^{K-1} \gamma^{K-k-1} \left((P^{\pi^{*}})^{K-k} +P^{\pi_{k}}(P^{\pi^{*}})^{K-k-1}\right)|\epsilon_{k}| +\nonumber\\
    && \gamma^{K} \Big((P^{\pi^{*}})^{K+1}+ P^{\pi_{k}}(P^{\pi^{*}})^{K}\Big)(Q^{*}-Q_{0})\Bigg]\label{thm_1_3}
\end{eqnarray}

 Equation \eqref{thm_1_2_11_2} is obtained from Equation \eqref{thm_1_2_11_1} by noting that the operator $\left(P^{\pi^{*}} + P^{\pi_{K}}\right)$ acting on $|\epsilon_{k}|$ will produce a larger value than $\left(P^{\pi^{*}} - P^{\pi_{K}}\right)$ acting on $|\epsilon_{k}|$.

\begin{equation}
    Q^{*}-Q^{\pi_{K}} 
    \leq (I-\gamma P^{\pi_{K}})^{-1} \left[\sum_{k=1}^{K-1} \gamma^{K-k-1} \left((P^{\pi^{*}})^{K-k} +P^{\pi_{k}}(P^{\pi^{*}})^{K-k-1}\right)|\epsilon_{k}| +  \gamma^{K} \left((P^{\pi^{*}})^{K+1} + P^{\pi_{k}}(P^{\pi^{*}})^{K} \right)(Q^{*}-Q_{0}) \right]\label{thm_1_4}
\end{equation}

We know $Q^{*} \leq \frac{R_{\max}}{1-\gamma}$ for all $(s,a) \in \mathcal{S}\times\mathcal{A}$ and $Q_{0}=0$ by initialization. Therefore, we have
\begin{equation}
    Q^{*}-Q^{\pi_{K}}
   \leq (I-\gamma P^{\pi_{K}})^{-1} \left[\sum_{k=1}^{K-1} \gamma^{K-k-1} \left((P^{\pi^{*}})^{K-k} +P^{\pi_{k}}(P^{\pi^{*}})^{K-k-1}\right)|\epsilon_{k}| +  \gamma^{K} \left((P^{\pi^{*}})^{K+1} + P^{\pi_{k}}(P^{\pi^{*}})^{K} \right)\frac{R_{\max}}{(1-\gamma)} \right]\label{thm_1_5}
\end{equation}

For simplicity, in stating our results, we also define

\begin{equation}
\alpha_{k} \triangleq\begin{cases}
          \frac{(1-\gamma)\gamma^{K-k-1}}{1-\gamma^{K+1}}  \quad &\text{if} \, 0  \leq k < K, \label{ps_5} \\
          \frac{(1-\gamma)\gamma^{K}}{1-\gamma^{K+1}} \quad &\text{if} \, k = K 
     \end{cases} 
\end{equation}

\begin{equation}
A_{k} \triangleq\begin{cases}
         \frac{(1-\gamma)}{2} (I-\gamma P^{\pi_{K}})^{-1}  \left( (P^{\pi^{*}})^{K-k} + P^{\pi_{K}}(P^{\pi^{*}})^{K-k-1} \right) \quad &\text{if} \, 0  \leq k < K, \label{ps_6}\\
       \frac{(1-\gamma)}{2} (I-\gamma P^{\pi_{K}})^{-1}  \left( (P^{\pi^{*}})^{K+1} + P^{\pi_{K}}(P^{\pi^{*}})^{K} \right) \quad &\text{if} \, k = K \\
     \end{cases} 
\end{equation}

We then substitute the value $\alpha_{k}$ and $A_{k}$ from Equation \eqref{ps_5} and \eqref{ps_6} respectively to get,

\begin{equation}
    Q^{*}-Q^{\pi_{K}} 
    \leq \frac{2\gamma(1-\gamma^{K+1})}{(1-\gamma)^{2}}\left[\sum_{k=1}^{K-1} \alpha_{k}A_{k}|\epsilon_{k}| +  \alpha_{K}A_{K}\frac{R_{\max}}{1-\gamma} \right]\label{thm_1_6}
\end{equation}

Taking an expectation on both sides with respect to  $\mu \in$ $\mathcal{P}(\mathcal{S} \times \mathcal{A})$ we get

\begin{equation}
    \mathbb{E}(Q^{*}-Q^{\pi_{K}})_{\mu} 
    \leq \frac{2\gamma(1-\gamma^{K+1})}{(1-\gamma)^{2}}\left[\sum_{k=1}^{K-1} \underbrace{\mathbb{E}(\alpha_{k}A_{k}|\epsilon_{k}|)_{\mu}}_{A} +  \underbrace{\alpha_{K}A_{K}\frac{R_{\max}}{1-\gamma}}_{B} \right]\label{thm_1_7}
\end{equation}

Consider the term $A$ in Equation \eqref{thm_1_7}
Plugging in the value of $\alpha_{k}A_{k}$ in the term $\mathbb{E}(\alpha_{k}A_{k}(\epsilon_{k}))_{\nu}$, we get
\begin{eqnarray}
    A &=& \mathbb{E}\left(\frac{(1-\gamma)^{2}\gamma^{K-k}}{2(1-\gamma^{K+1})} (I-\gamma P^{\pi_{k}})^{-1}\left( (P^{\pi^{*}})^{K-k} + (P^{\pi^{*}})^{K-k-1}P^{\pi_{k}}\right)|\epsilon_{k}|\right)_{\mu} \label{thm_1_8}\\
      &=& \mathbb{E}\left(\frac{(1-\gamma)^{2}\gamma^{K-k}}{2(1-\gamma^{K+1})} \sum_{m=1}^{\infty}(\gamma{P}^{\pi_{k}})^{m}\left( (P^{\pi^{*}})^{K-k} + (P^{\pi^{*}})^{K-k-1}P^{\pi_{k}}\right)|\epsilon_{k}|\right)_{\mu} \label{thm_1_9}\\
      &=& \mathbb{E}\left(\frac{(1-\gamma)^{2}\gamma^{K-k}}{2(1-\gamma^{K+1})} \sum_{m=1}^{\infty}\left( (\gamma{P}^{\pi_{k}})^{m}(P^{\pi^{*}})^{K-k} +(\gamma{P}^{\pi_{k}})^{m+1}(P^{\pi^{*}})^{K-k-1}\right)|\epsilon_{k}|\right)_{\mu} \label{thm_1_10}\\
       &=&\frac{(1-\gamma)^{2}\gamma^{K-k}}{2(1-\gamma^{K+1})} \sum_{m=1}^{\infty}\left(\underbrace{\mathbb{E}\left((\gamma{P}^{\pi_{k}})^{m}(P^{\pi^{*}})^{K-k}|\epsilon_{k}|\right)_{\mu}}_{A1}  +\underbrace{\mathbb{E}\left((\gamma{P}^{\pi_{k}})^{m+1}(P^{\pi^{*}})^{K-k-1}|\epsilon_{k}|\right)_{\mu}}_{A2}\right) \label{thm_1_11}\\
\end{eqnarray}

Equation \eqref{thm_1_9} is obtained from \eqref{thm_1_8} by writing $(I-\gamma P^{\pi_{k}})^{-1} = \sum_{m=1}^{\infty} (\gamma{P^{\pi_{k}}})^{m}$ using the binomial expansion formula.

Consider the term $A1$ in Equation \eqref{thm_1_10}, We have
\begin{eqnarray}
    A1 &=& \mathbb{E}\left((\gamma{P}^{\pi_{k}})^{m}(P^{\pi^{*}})^{K-k}|\epsilon_{k}|\right)_{\mu} \label{thm_1_12}\\
       &=&  \int_{\mathcal{S}\times\mathcal{A}} (\gamma{P}^{\pi_{k}})^{m}(P^{\pi^{*}})^{K-k}|\epsilon_{k}|d\mu(s,a) \label{thm_1_13}\\
       &\le&  \int_{\mathcal{S}\times\mathcal{A}}  \frac{d\left( (\gamma{P}^{\pi_{k}})^{m}(P^{\pi^{*}})^{K-k}\mu\right)}{d\nu} |\epsilon_{k}|  d\nu(s,a)  \label{thm_1_14}\\
       &\le&  \gamma^{m}\phi_{\nu,\mu}\left(\int_{\mathcal{S}\times\mathcal{A}}|\epsilon_{k}| d\nu(s,a)\right) \label{thm_1_15}\\
       &\le&  \gamma^{m}\phi_{\nu,\mu}\mathbb{E}(|\epsilon_{k}|)_{\nu}  \label{thm_1_16}
\end{eqnarray}

Here, Equation  \eqref{thm_1_15} is obtained from \eqref{thm_1_14} from Assumption \ref{assump_6}. 

In the same manner, we have 

\begin{eqnarray}
    A2 &\le&  \gamma^{m}\phi_{\nu,\mu}\mathbb{E}(|\epsilon_{k}|)_{\nu}  \label{thm_1_17}
\end{eqnarray}

Denote the term in Equation \eqref{thm_1_13} as 

\begin{equation}
    \int_{\mathcal{S}\times\mathcal{A}} (\gamma{P}^{\pi_{k}})^{m}(P^{\pi^{*}})^{K-k}(|\epsilon_{k}|)d\mu(s,a) = \int_{\mathcal{S}\times\mathcal{A}}(|\epsilon_{k}|)d\tilde{\nu} 
\end{equation}

Where $({P}^{\pi_{k}})^{m}(P^{\pi^{*}})^{K-k}\mu$ is the marginal distribution of the state action pair at step $m+K-k+1$ denoted by $\tilde{\nu}$ for notational simplicity. It is the state action distribution at step $m+K-k+1$ obtained by starting the state action pair sampled form $\mu$ and the following the policies $\pi^{*}$ for $K-k$ steps and then $\pi_{k}$ for $m$ steps. We then get Equation \eqref{thm_1_14} from \eqref{thm_1_13} from the definition of the Radon Nikodym derivative as follows.

\begin{equation}
    \int_{\mathcal{S}\times\mathcal{A}}|\epsilon_{k}|d\tilde{\nu}  = \int_{\mathcal{S}\times\mathcal{A}}\frac{d\tilde{\nu}}{d{\nu}}|\epsilon_{k}|d\nu
\end{equation}

Plugging upper bound on $A1$  and $A2$ into Equation \eqref{thm_1_11}, we obtain

\begin{eqnarray}
    A &\le& \frac{(1-\gamma)^{2}\gamma^{K-k}}{2(1-\gamma^{K+1})} \sum_{m=1}^{\infty}\left(2\phi_{\nu,\mu}\gamma^{m}\mathbb{E}(|\epsilon_{k}|)_{\nu}\right) \label{thm_1_18}\\
      &\le& \frac{(1-\gamma)^{2}\gamma^{K-k}}{2(1-\gamma^{K+1})} \left(\frac{2\phi_{\nu,\mu}\mathbb{E}(|\epsilon_{k}|)_{\nu}}{1-\gamma}\right) \label{thm_1_19}\\
      &\le& \frac{\phi_{\nu,\mu}(1-\gamma)\gamma^{K-k}}{(1-\gamma^{K+1})} \mathbb{E}(|\epsilon_{k}|)_{\nu} \label{thm_1_20}\\
      &\le& \frac{\phi_{\nu,\mu}\gamma^{K-k}}{(1-\gamma^{K+1})} \mathbb{E}(|\epsilon_{k}|)_{\nu} 
\end{eqnarray}

Now consider $B$ in Equation \eqref{thm_1_7}. We have

\begin{eqnarray}
    B &=& \alpha_{K}A_{k}\frac{R_{max}}{1-\gamma} \label{thm_1_21}\\
      &=& \alpha_{K}(1-\gamma)\frac{R_{max}}{1-\gamma} \label{thm_1_22}\\
      &=& \frac{\gamma^{K}R_{max}}{(1-\gamma^{K+1})} \label{thm_1_22_1}
\end{eqnarray}

Equation \eqref{thm_1_22} is obtained from the fact that  the transition operator $A_{k}$ acting on the constant $\frac{R_{max}}{1-\gamma}$ is equivalent to multiplying it by ${1-\gamma}$. Equation \eqref{thm_1_22_1} is obtained by plugging in the value of $\alpha_{K}$ from  Equation \eqref{ps_5}. Plugging upper bound on $A$  and value of $B$ into Equation \eqref{thm_1_7}, we obtain

\begin{eqnarray}
    \mathbb{E}(Q^{*}-Q^{\pi_{K}})_{\mu} 
    \leq \frac{2\gamma(1-\gamma^{K+1})}{(1-\gamma)^{2}}\left[\sum_{k=1}^{K-1} \frac{\phi_{\nu,\mu}\gamma^{K-k}}{(1-\gamma^{K+1})} \mathbb{E}(|\epsilon_{k}|)_{\nu}  +   \frac{\gamma^{K}R_{max}}{2(1-\gamma^{K+1})}   \right]\label{thm_1_23}\\
     \mathbb{E}(Q^{*}-Q^{\pi_{K}})_{\mu} 
    \leq \frac{2\phi_{\nu,\mu}\gamma}{(1-\gamma)}\left[\sum_{k=1}^{K-1} \gamma^{K-k}\mathbb{E}(|\epsilon_{k}|)_{\nu}\right] +  \frac{2R_{\max}\gamma^{K+1}}{(1-\gamma)^{2}} \label{thm_1_24}
\end{eqnarray}

From Equation \eqref{last} we have that 

\begin{eqnarray}
   \epsilon_{k} &=& TQ_{k-1} - Q_{k} \nonumber\\
                 &=& \epsilon_{k_1} + \epsilon_{k_2} +\epsilon_{k_3} +\epsilon_{k_4}  \label{thm_1_25}
\end{eqnarray}

Thus, equation \eqref{thm_1_24} becomes
\begin{eqnarray}
    E{(Q^{*}-Q^{\pi_{K}})}_{\mu} 
    &\leq& \frac{2\gamma}{(1-\gamma)}\left[\sum_{k=1}^{K-1} \gamma^{K-k}\mathbb{E}(|\epsilon_{k_1} + \epsilon_{k_2} +\epsilon_{k_3} +\epsilon_{k_4}|)_{\mu}\right] +  \frac{2R_{\max}\gamma^{K+1}}{(1-\gamma)^{3}} \label{thm_1_26} \\
    &\leq& \underbrace{\frac{2\gamma}{(1-\gamma)}\left[\sum_{k=1}^{K-1}\gamma^{K-k}\mathbb{E}(|\epsilon_{k_1}|)_{\mu}\right]}_{T1} + \underbrace{\frac{2\gamma}{(1-\gamma)}\left[\sum_{k=1}^{K-1} \gamma^{K-k}\mathbb{E}(|\epsilon_{k_2}|)_{\mu}\right]}_{T2} +\underbrace{\frac{2\gamma}{(1-\gamma)}\left[\sum_{k=1}^{K-1} \gamma^{K-k}\mathbb{E}(|\epsilon_{k_3}|)_{\mu}\right]}_{T3} +\nonumber\\
    && \underbrace{\frac{2\gamma}{(1-\gamma)}\left[\sum_{k=1}^{K-1} \gamma^{K-k}\mathbb{E}(|\epsilon_{k_4}|)_{\mu}\right]}_{T4} +\underbrace{\frac{2R_{\max}\gamma^{K+1}}{(1-\gamma)^{2}}}_{T5}  \label{thm_1_27}
\end{eqnarray}

Consider $T1$, from Lemma \ref{lem_1}, we have $\mathbb{E}(|\epsilon_{k_1}|)_{\nu} \le \sqrt{\epsilon_{bias}}$. Thus, for term $T1$, we obtain  

\begin{eqnarray}
    T1 &=& \frac{2\phi_{\nu,\mu}\gamma}{(1-\gamma)}\left[\sum_{k=1}^{K-1}\gamma^{K-k}\mathbb{E}(|\epsilon_{k_1}|)_{\nu}\right] \label{thm_1_28}\\
    &\le& \frac{2\phi_{\nu,\mu}\gamma}{(1-\gamma)} \left[\sum_{k=1}^{K-1}\gamma^{K-k}\sqrt{\epsilon_{bias}}\right] \label{thm_1_29}\\ 
    &\le& \frac{2\phi_{\nu,\mu}\sqrt{\epsilon_{bias}}\gamma}{(1-\gamma)}\frac{1}{1-\gamma}\label{thm_1_30}\\
    T1 &\le& \frac{2\phi_{\nu,\mu}\gamma\sqrt{\epsilon_{bias}}}{(1-\gamma)^{2}}\label{thm_1_31}
\end{eqnarray}

Equation \eqref{thm_1_30} is obtained from \eqref{thm_1_29} by using the inequality $\sum_{k=1}^{K-1}\gamma^{K-k} \le \frac{1}{1-\gamma}$. 

Consider $T2$ we have from Lemma \ref{lem_2} $\epsilon_{k2}=0$.

Therefore 
\begin{equation}
T2=0 \label{thm_1_32}     
\end{equation}

Consider $T3$, from Lemma \ref{lem_3}, we have that if the number of samples of state action pairs at iteration $k$ denoted by $n_{k}$ satisfy

\begin{eqnarray}
   n_{k} \ge 8\left({(C_{k}{\eta}\beta_{k}})^{2}{\epsilon}^{-2}\right)  \label{thm_1_33}
\end{eqnarray}

then we have
    
    \begin{equation}
        \mathbb{E}\left(\epsilon_{k_3}\right)_{\nu} \le  \epsilon
    \end{equation}

Plugging in $\epsilon = \epsilon\frac{(1-\gamma)^{2}}{6\phi_{\nu,\mu}\gamma}$ in Equation \eqref{thm_1_33},   we get that if number of samples of state action pairs at iteration $k$ denoted by $n_{k}$ and the step size $\alpha$ satisfy

\begin{eqnarray}
    n_{k} &\ge& 288(C_{k}{\eta}\beta_{k}\phi_{\nu,\mu}\gamma)^{2} (1-\gamma)^{-4}{\epsilon}^{-2}          \label{thm_1_33_1}\\
    &&\forall k \in \{1,\cdots,K\} \nonumber
    \alpha = 
\end{eqnarray}

then we have

\begin{eqnarray}
    T3 &=& \frac{2\phi_{\nu,\mu}\gamma}{(1-\gamma)}\left[\sum_{k=1}^{K-1}\gamma^{K-k}\mathbb{E}(|\epsilon_{k_3}|)_{\nu}\right] \label{thm_1_34}\\
    &\le&  \frac{2\phi_{\nu,\mu}\gamma}{(1-\gamma)}\left[\sum_{k=1}^{K-1}\gamma^{K-k}\epsilon \frac{(1-\gamma)^{2}}{6\phi_{\nu,\mu}\gamma}\right]  \label{thm_1_35}\\ 
    &\le&  \frac{2\gamma\epsilon}{(1-\gamma)}\left[\sum_{k=1}^{K-1}\gamma^{K-k}\frac{(1-\gamma)^{2}}{6\gamma}\right]  \label{thm_1_28_1}\\
    &\le&  \frac{\epsilon}{3(1-\gamma)}\left[\sum_{k=1}^{K-1}\gamma^{K-k}(1-\gamma)^{2}\right]   \label{thm_1_36}\\
    &\le&  \frac{\epsilon}{3(1-\gamma)}\left[\frac{1}{1-\gamma}(1-\gamma)^{2}\right]  \label{thm_1_37}\\
    &\le&  \frac{\epsilon}{3}(1-\gamma^{K}) \label{thm_1_38}\\
    T3 &\le&  \frac{\epsilon}{3} \label{thm_1_39}
\end{eqnarray}

Equation \eqref{thm_1_37} is obtained from \eqref{thm_1_36} by using the inequality $\sum_{k=1}^{K-1}\gamma^{K-k} \le \frac{1}{1-\gamma}$.

Consider $T4$. From Lemma \ref{lem_4}, we have that if the number of iterations of the projected gradient descent algorithm at iteration $k$ denoted by $T_{k}$ and the step size denoted by $\alpha_{k}$ satisfy

\begin{eqnarray}
     T_{k} &\ge& \left({\frac{C^{'}_{k}\epsilon}{l_{k}}}\right)^{-2}L_{k}^{2}||u_{k}^{*}||^{2}_{2} - 1 \label{thm_1_40}\\
     \alpha_{k} &=& \frac{||u^{*}_{k}||_{2}}{L_{k}\sqrt{T_{k}+1}}\\
\end{eqnarray}

and $\epsilon < C^{'}_{k}$, then we have 

\begin{equation}
       \mathbb{E}(|\epsilon_{k4}|)_{\nu} \le \epsilon + \epsilon_{|\tilde{D}|} \label{thm_1_41}
\end{equation}
Plugging in $\epsilon = \epsilon \frac{(1-\gamma)^{2}}{6\phi_{\nu,\mu}\gamma}$ in Equation \eqref{thm_1_40} we get that if the number of iterations of the projected gradient descent algorithm at iteration $k$ denoted by $T_{k}$ and the step size denoted by $\alpha_{k}$ satisfy

\begin{eqnarray}
     T_{k} &\ge& \left({\frac{C^{'}_{k}\epsilon(1-\gamma)^{2}}{6l\phi_{\nu,\mu}\gamma}}\right)^{-2}L_{k}^{2}||u_{k}^{*}||^{2}_{2} - 1 \label{thm_1_42}\\ 
     \alpha_{k} &=& \frac{||u^{*}_{k}||_{2}}{L_{k}\sqrt{T_{k}+1}}\\
          &&\forall k \in \{1,\cdots,K\} \nonumber
\end{eqnarray}

and if $\epsilon < C^{'}_{k}$ for all $k \in (1,\cdots,K)$, then we have

\begin{eqnarray}
    T4 &=& \frac{2\phi_{\nu,\mu}\gamma}{(1-\gamma)}\left[\sum_{k=1}^{K-1}\gamma^{K-k}\mathbb{E}(|\epsilon_{k_{4}}|)_{\nu}\right]   \label{thm_1_43}\\
    &\le&  \frac{2\phi_{\nu,\mu}\gamma}{(1-\gamma)}\left[\sum_{k=1}^{K-1}\gamma^{K-k}\epsilon \frac{(1-\gamma)^{2}}{6\phi_{\nu,\mu}\gamma}\right] + \frac{2\phi_{\nu,\mu}\gamma}{(1-\gamma)}\left[\sum_{k=1}^{K-1}\gamma^{K-k}\epsilon_{|\tilde{D}|}\right] \label{thm_1_44}\\ 
    &\le&  \frac{2\gamma\epsilon}{(1-\gamma)}\left[\sum_{k=1}^{K-1}\gamma^{K-k}\frac{(1-\gamma)^{2}}{6\gamma}\right] +\frac{2\phi_{\nu,\mu}\epsilon_{|\tilde{D}|}\gamma}{(1-\gamma)}\frac{1}{1-\gamma} \label{thm_1_45}\\
    &\le&  \frac{\epsilon}{3(1-\gamma)}\left[\sum_{k=1}^{K-1}\gamma^{K-k}(1-\gamma)^{2}\right] + \frac{2\phi_{\nu,\mu}\epsilon_{|\tilde{D}|}}{(1-\gamma)^{2}}\label{thm_1_46}\\
    &\le&  \frac{\epsilon}{3(1-\gamma)}\left[\frac{1}{1-\gamma}(1-\gamma)^{2}\right] +  \frac{2\phi_{\nu,\mu}\epsilon_{|\tilde{D}|}}{(1-\gamma)^{2}} \label{thm_1_47}\\
    &\le&  \frac{\epsilon}{3}(1-\gamma^{K}) + \frac{2\phi_{\nu,\mu}\epsilon_{|\tilde{D}|}}{(1-\gamma)^{2}}\label{thm_1_48}\\
    &\le&  \frac{\epsilon}{3} +\frac{2\phi_{\nu,\mu}\epsilon_{|\tilde{D}|}}{(1-\gamma)^{2}}\label{thm_1_49}
\end{eqnarray}

Equation \eqref{thm_1_44}, \eqref{thm_1_47}  are  obtained from \eqref{thm_1_43}, \eqref{thm_1_46},  respectively, by using the inequality $\sum_{k=1}^{K-1}\gamma^{K-k} \le \frac{1}{1-\gamma}$.

Consider $T5$. Assume $K$ is large enough such that

\begin{eqnarray}
        K &\geq& \frac{1}{\log\left(\frac{1}{\gamma}\right)}\log\left(\frac{6R_{max}}{\epsilon(1-\gamma)^{2}}\right) -1  \label{thm_1_54}\\
        {(K+1)}{\log\left(\frac{1}{\gamma}\right)} &\geq& \log\left(\frac{6R_{max}}{\epsilon(1-\gamma)^{2}}\right) \label{thm_1_53}\\
        {(K+1)}{\log(\gamma)} &\le& \log\left(\frac{\epsilon(1-\gamma)^{2}}{6R_{max}}\right) \label{thm_1_52}\\
        \gamma^{K+1} &\le& \frac{\epsilon(1-\gamma)^{2}}{6R_{max}} \label{thm_1_51}\\
       \frac{2R_{max}\gamma^{K+1}}{(1-\gamma)^{2}} &\le& \frac{\epsilon}{3} \label{thm_1_50}   
\end{eqnarray}

Equation \eqref{thm_1_52} is obtained from \eqref{thm_1_53} by multiplying on $-1$ on both sides and noting that $\log(x) = -\log\left(\frac{1}{x}\right)$

Thus, we obtain that if the number of iterations of Algorithm \ref{algo_1}, number of state actions pairs sampled at iteration $k$ denoted by $n_{k}$ and the number of iterations of the projected gradient descent on iteration $k$ denoted by $T_{k}$ and the step size denoted by $\alpha_{k}$ satisfy    

\begin{eqnarray}
    K &\geq& \frac{1}{\log\left(\frac{1}{\gamma}\right)}\log\left(\frac{6R_{max}}{\epsilon(1-\gamma)^{2}}\right) -1 \label{thm_1_55} \\  
    n_{k} &\ge& 288(C_{k} {\eta}\beta_{k}\phi_{\nu,\mu}\gamma)^{2} (1-\gamma)^{-4}{\epsilon}^{-2}       \label{thm_1_56}\\
     T_{k} &\ge& \left({\frac{C^{'}_{k}\epsilon(1-\gamma)^{2}}{6l_{k}\phi_{\nu,\mu}\gamma}}\right)^{-2}L_{k}^{2}||u_{k}^{*}||^{2}_{2} - 1 \label{thm_1_57}\\
    \alpha_{k} &=& \frac{||u^{*}_{k}||_{2}}{L_{k}\sqrt{T_{k}+1}}\\
    && \forall k \in \{1,\cdots,K\} \nonumber
\end{eqnarray} 

and if $\epsilon < (C^{'}_{k})$ for all $k \in (1,\cdots,K)$,

We have 
\begin{eqnarray}
    E{(Q^{*}-Q^{\pi_{K}})}_{\mu} \le \epsilon + \frac{2\gamma}{(1-\gamma)^{2}}(\sqrt{\epsilon_{bias}} + \epsilon_{\tilde{D}})
    \label{thm_1_58} 
\end{eqnarray}

\end{proof}

\section{Proof of Supporting Lemmas} \label{Upper Bounding Bellman Error}

\subsection{Proof Of Lemma \ref{lem_1}} \label{proof_lem_1}

\begin{proof}

Using Assumption \ref{assump_3} and the definition of $Q_{k1}$ for some iteration $k$ of Algorithm \ref{algo_1} we have  
\begin{equation}
\mathbb{E}(TQ_{k-1} - Q_{k1})^{2}_{\nu} \le \epsilon_{bias}  
\end{equation}

Since $|a|^{2}=a^{2}$ we obtain

\begin{equation}
\mathbb{E}(|TQ_{k-1} - Q_{k1}|)^{2}_{\nu} \le \epsilon_{bias}  
\end{equation}

We have for a random variable $x$, $ Var(x)=\mathbb{E}(x^{2}) - (\mathbb{E}(x))^{2}$ hence $\mathbb{E}(x) = \sqrt{\mathbb{E}(x^{2}) -Var(x)}$, Therefore replacing $x$ with $|TQ_{k-1} - Q_{k1}|$ we get

using the definition of the variance of a random variable we get  
\begin{equation}
\mathbb{E}(|TQ_{k-1} - Q_{k1}|)_{\nu}=\sqrt{\mathbb{E}(|TQ_{k-1} - Q_{k1}|)^{2}_{\nu} - Var(|TQ_{k-1} - Q_{k1}|)_{\nu}}  
\end{equation}

Therefore we get

\begin{equation}
    \mathbb{E}(|TQ_{k-1} - Q_{k1}|)_{\nu} \le \sqrt{\epsilon_{bias}}
\end{equation}

Since $\epsilon_{k_1}=TQ_{k-1} - Q_{k1}$ we have 
\begin{equation}
    \mathbb{E}(|\epsilon_{k_1}|)_{\nu} \le \sqrt{\epsilon_{bias}}
\end{equation}

$\blacksquare$
\end{proof}

\subsection{Proof Of Lemma \ref{lem_2}} \label{proof_lem_2}
\begin{proof}
From Lemma \ref{sup_lem_1}, we have 

\begin{equation}
    \argminA_{f_{\theta}}\mathbb{E}_{x,y}\left(f_{\theta}(x)-g(x,y)\right)^{2}=\argminA_{f_{\theta}} \left(\mathbb{E}_{x,y}\left(f_{\theta}(x)-\mathbb{E}(g(y^{'},x)|x)\right)^{2}\right) \label{lem_3_1}
\end{equation}   

The function $f_{\theta}(x)$ to be $Q_{\theta}(s,a)$ and $g(x,y)$ to be the function $r^{'}(s,a) + \max_{a^{'} \in \mathcal{A}}{\gamma}Q_{k-1}(s^{'},a^{'})$.

We also have $y$ as the two dimensional random variable $(r^{'}(s,a),s^{'})$. We now have $(s,a) \sim \nu$ and $s^{'}|(s,a) \sim P(.|(s,a))$ and $r^{'}(s,a) \sim {R}(.|s,a)$.

Then the loss function in \eqref{sup_lem_1_1} becomes
\begin{equation}
\mathbb{E}_{(s,a) \sim \nu,s^{'} \sim P(s^{'}|s,a), r(s,a) \sim \mathcal{R}(.|s,a)}(Q_{\theta}(s,a)-(r^{'}(s,a)+\max_{a^{'}}{\gamma}Q_{k-1}(s^{'},a^{'})))^{2} \label{lem_3_2}
\end{equation}

Therefore by Lemma \ref{sup_lem_1}, we have that the function $Q_{\theta}(s,a)$ which minimizes Equation \eqref{lem_3_2} it will be minimizing

\begin{equation}
\mathbb{E}_{(s,a) \sim \nu}(Q_{\theta}(s,a) -\mathbb{E}_{s^{'} \sim P(s^{'}|s,a),r \sim \mathcal{R}(.|s,a))}(r^{'}(s,a)+\max_{a^{'}}{\gamma}Q_{k-1}(s^{'},a^{'})|s,a))^{2} \label{lem_3_3}
\end{equation}

But we have from Equation \eqref{ps_4} that

\begin{equation}
\mathbb{E}_{s^{'} \sim P(s^{'}|s,a),r \sim {R}(.|s,a))}(r^{'}(s,a)+\max_{a^{'}}{\gamma}Q_{k-1}(s^{'},a^{'})|s,a) = TQ_{k-1}\label{lem_3_4}
\end{equation}

Combining Equation \eqref{lem_3_2} and \eqref{lem_3_4} we get

\begin{equation}
\argminA_{Q_{\theta}}\mathbb{E}_{(s,a) \sim \nu,s^{'} \sim P(s^{'}|s,a),r \sim {R}(.|s,a))}(Q_{\theta}(s,a)-(r(s,a)+\max_{a^{'}}{\gamma}Q_{k-1}(s^{'},a^{'})))^{2} = \argminA_{Q_{\theta}}\mathbb{E}_{(s,a) \sim \nu}(Q_{\theta}(s,a)-TQ_{k-1})^{2} \label{lem_3_4_1}
\end{equation}

The left hand side of Equation \eqref{lem_3_4_1} is $Q_{k_2}$ as defined in Definition \ref{def_2} and the right hand side is  $Q_{k_1}$ as defined in Definition \ref{def_1}, which gives us 

\begin{equation}
    Q_{k_2}=Q_{k_1}
\end{equation}

\end{proof}

\subsection{Proof Of Lemma \ref{lem_3}}\label{proof_lem_3}

\begin{proof}
 
We define $R_{X,Q_{k-1}}({\theta})$ as

\begin{equation}
    R_{X,Q_{k-1}}({\theta}) = \frac{1}{|X|} \sum_{(s_{i},a_{i}) \in X}\Bigg( Q_{\theta}(s_{i},a_{i}) - \Bigg(r^{'}(s_{i},a_{i})\nonumber\\ 
        + \gamma\max_{a^{'} \in \mathcal{A}}Q_{k-1}(s_{i}^{'},a^{'}) \Bigg)\Bigg)^{2},
\end{equation}

Here, $X=\{s_{i},a_{i}\}_{i=\{1,\cdots,|X|\}}$, where $s_{i}, a_{i} \sim \nu \in \mathcal{P}(\mathcal{S}\times\mathcal{A})$, $r(s_{i}, a_{i}) \sim R(.|s_{i}, a_{i})$ and $s^{'}_{i} \sim P(.|s_{i},a_{i})$. $\theta \in \Theta$, $Q_{\theta}$ is as defined in Equation \eqref{ReLU_1_2} and $Q_{k-1}$ is the estimate of the $Q$ function obtained at iteration $k-1$ of Algorithm \ref{algo_1}.

We also define the term

\begin{equation}
   L_{Q_{k-1}}({\theta}) = \mathbb{E}_{(s,a) \sim \nu ,s' \sim P(s'|s,a), r'(\cdot|s,a)\sim {R}(\cdot|s,a)}\nonumber\\
    (Q_{\theta}(s,a)-(r'(s,a)+\gamma\max_{a'}Q_{k-1}(s',a'))^{2}  
\end{equation}

We denote by $\theta_{k_{2}}, \theta_{k_{3}}$ the parameters of the neural networks $Q_{k_{2}}, Q_{k_{3}}$ respectively. $Q_{k_{2}}, Q_{k_{3}}$ are defined in Definition \ref{def_2} and  \ref{def_3} respectively.  

We then obtain,

\begin{eqnarray}
    R_{X,Q_{k-1}}(\theta_{k_{2}}) - R_{X,Q_{k-1}}(\theta_{k_{3}}) &\le&  R_{X,Q_{k-1}}(\theta_{k_{2}}) - R_{X,Q_{k-1}}(\theta_{k_{3}}) + L_{Q_{k-1}}(\theta_{k_{2}}) - L_{Q_{k-1}}(\theta_{k_{3}})  \label{2_2_2}\\
                                                        &=&   {R_{X,Q_{k-1}}(\theta_{k_{2}})- L_{Q_{k-1}}(\theta_{k_{2}})} - {R_{X,Q_{k-1}}(\theta_{k_{3}}) + L_{Q_{k-1}}(\theta_{k_{2}})}  \label{2_2_3}\\
                                                        &\le&   \underbrace{|R_{X,Q_{k-1}}(\theta_{k_{2}})- L_{Q_{k-1}}(\theta_{k_{2}})|}_{I} + \underbrace{|R_{X,Q_{k-1}}(\theta_{k_{3}})- L_{Q_{k-1}}(\theta_{k_{3}})|}_{II}  \label{2_2_4}
\end{eqnarray}

We get the inequality in Equation \eqref{2_2_2} because $L_{Q_{k-1}}(\theta_{k_{3}}) - L_{Q_{k-1}}(\theta_{k_{2}}) > 0$ as  $Q_{k_{2}}$ is the minimizer of the loss function $ L_{Q_{k-1}}(Q_{\theta})$.

Consider Lemma \ref{sup_lem_0}. The loss function  $R_{X_{k},Q_{k-1}}(\theta_{k_{3}})$ can be written as the mean of loss functions of the form $l(a_{\theta}(s_{i},a_{i}),y_{i})$ where $l$ is the square function. $a_{\theta}(s_{i}, a_{i})=Q_{\theta}(s_{i},a_{i})$  and $y_{i}=\Big(r^{'}(s_{i},a_{i}) + \gamma\max_{a^{'} \in \mathcal{A}}Q_{k-1}(s^{'},a^{'})\Big)$. Thus we have

\begin{equation}
   \mathbb{E}\sup_{\theta \in \Theta}|R_{X,Q_{k-1}}(\theta)- L_{Q_{k-1}}(\theta)| \le 2{\eta}\mathbb{E} \left(Rad(\mathcal{A} \circ \{(s_{1},a_{1}),(s_{2},a_{2}),(s_{3},a_{3}),\cdots,(s_{n},a_{n})\})\right)
\end{equation}

Where  $n_{k}=|X|$, $(\mathcal{A} \circ \{(s_{1},a_{1}),(s_{2},a_{2}),(s_{3},a_{3}),\cdots,(s_{n},a_{n})\} = \{Q_{\theta}(s_{1},a_{1}), Q_{\theta}(s_{2},a_{2}), \cdots, Q_{\theta}(s_{n},a_{n})\}$  and $\eta$ is the Lipschitz constant for the square function  over the state action space $[0,1]^{d}$. The expectation is with respect to ${{(s_{i},a_{i}) \sim \nu,{s_{i}^{'} \sim P(s^{'}|s,a)},r_{i} \sim R(.|s_{i},a_{i})}_{_{i \in (1,\cdots,n_{k})}, }}$. 

From (Ma,Tengyu.(2018). Statistical Learning Theory [Lecture Notes]. https://web.stanford.edu/class/cs229t/scribe{\textunderscore}notes  we have that 

\begin{equation}
    \left(Rad(\mathcal{A} \circ \{(s_{1},a_{1}),(s_{2},a_{2}),(s_{3},a_{3}),\cdots,(s_{n},a_{n})\})\right) \le 2\frac{\beta_{k}}{\sqrt{n_{k}}}
\end{equation}

Where $\{||\theta\||_{2}  \le \beta_{k}; \forall \theta \in \Theta\}$. Since we only have to demonstrate the inequality for $\theta_{k_{2}}$ and $\theta_{k_{3}}$, we set $\beta_{k}$ as  a constant greater than $\max(||\theta_{k_{2}}||_{2},||\theta_{k_{3}}||_{2})$, which  gives us

\begin{eqnarray}
   \mathbb{E}|(R_{X,Q_{k-1}}(\theta_{k_{2}})) - L_{Q_{k-1}}(\theta_{k_{2}})| \le  2\left(\frac{{\eta}\beta_{k}}{\sqrt{n_{k}}}\right) \label{2_2_5}
\end{eqnarray}

The same argument can be applied for $Q_{k_3}$ to get
\begin{eqnarray}
   \mathbb{E}|(R_{X,Q_{k-1}}(\theta_{k_{3}})) - L_{Q_{k-1}}(\theta_{k_{3}})| \le  2\left(\frac{{\eta}\beta_{k}}{\sqrt{n_{k}}}\right) \label{2_2_5_1}
\end{eqnarray}

Thus, if we have  

\begin{equation}
    n_{k} \ge 2\left(\frac{2{\eta}\beta_{k}}{\epsilon}\right)^{2}
\end{equation}

Then we have 

\begin{equation}
   \mathbb{E}\left(R_{X,Q_{k-1}}(\theta_{k_{2}}) - R_{X,Q_{k-1}}(\theta_{k_{3}})\right) \le {\epsilon} \label{2_2_5_2}
\end{equation}

Plugging in the definition of $R_{X,Q_{k-1}}(\theta_{k_{2}}), R_{X,Q_{k-1}}(\theta_{k_{3}})$ in equation  \eqref{2_2_5_2}  we get

\begin{eqnarray}
    \mathbb{E} \frac{1}{n_{k}} \sum_{i=1}^{n_{k}} (Q_{k_{2}}(s_{i},a_{i})-Q_{k_{3}}(s_{i},a_{i}))(Q_{k_{2}}(s_{i},a_{i})+Q_{k_{3}}(s_{i},a_{i}) - 2(r(s_{i},a_{i})) + \gamma \max_{a \in \mathcal{A}}Q_{k-1}(s^{'},a)) &\le& \epsilon \label{2_2_5_3}
\end{eqnarray}
 
Now since all $(s_{i},a_{i})$ are independent \eqref{2_2_5_3} becomes

\begin{eqnarray}
    \mathbb{E}\underbrace{(Q_{k_{2}}(s,a)-Q_{k_{3}}(s,a))}_{A1}\underbrace{(Q_{k_{2}}(s,a)+Q_{k_{3}}(s,a) - 2(r(s,a)) + \gamma \max_{a \in \mathcal{A}}Q_{k-1}(s^{'},a))}_{A2} &\le& \epsilon \label{2_2_5_4}
\end{eqnarray}

Where the expectation is now over a single $(s,a)$ drawn from $\nu$, $r(s,a)\sim R(.|s,a)$ and $s^{'}\sim P(.|s,a)$.  
We re-write Equation \eqref{2_2_5_4} as 

\begin{eqnarray}
    \int\underbrace{(Q_{k_{2}}(s,a)-Q_{k_{3}}(s,a))}_{A1}\underbrace{(Q_{k_{2}}(s,a)+Q_{k_{3}}(s,a) - 2(r(s,a)) + \gamma \max_{a \in \mathcal{A}}Q_{k-1}(s^{'},a))}_{A2}d{\nu}(s,a)d{\mu_{2}}(r)d{\mu_{3}}(s^{'}) &\le& \epsilon \nonumber\\
    \label{2_2_5_5}
\end{eqnarray}

Where $\nu$, $\mu_{2}$, $\mu_{3}$ are the measures with respect to  $(s,a)$, $r^{'}$ and $s^{'}$ respectively

Now for the integral in Equation \eqref{2_2_5_5} we split the integral into four different integrals. Each integral is over the set of $(s,a), r^{'}, s^{'}$ corresponding to the 4 different combinations of signs of $A1, A2$. 

\begin{eqnarray}
    \int_{\{(s,a), r^{'}, s^{'}\}:A1\ge0,A2\ge0}(A1)(A2)d{\nu}(s,a)d{\mu_{2}}(r)d{\nu_{3}}(s^{'}) + \int_{\{(s,a), r^{'}, s^{'}\}:A1<0,A2<0}(A1)(A2)d{\nu}(s,a)d{\mu_{2}}(r)d{\mu_{3}}(s^{'}) + \nonumber\\
    \int_{\{(s,a), r^{'}, s^{'}\}:A1\ge0,A2<0}(A1)(A2)d{\nu}(s,a)d{\mu_{2}}(r)d{\nu_{3}}(s^{'}) +  \int_{\{(s,a), r^{'}, s^{'}\}:A1<0,A2\ge0}(A1)(A2)d{\nu}(s,a)d{\mu_{2}}(r)d{\mu_{3}}(s^{'})  &\le& \epsilon \nonumber\\
    \label{2_2_5_6}
\end{eqnarray}

Now note that the first 2 terms are non-negative and the last two terms are non-positive. We then write the first two terms as 

\begin{eqnarray}
    \int_{\{(s,a), r^{'}, s^{'}\}:A1\ge0,A2\ge0}(A1)(A2)d(s,a)d{\nu}(s,a)d{\mu_{2}}(r)d{\mu_{3}}(s^{'}) = C_{k_{1}}\int|Q_{k_{2}}-Q_{k_{3}}|d{\nu}(s,a) =  C_{k_{1}}\mathbb{E}(|Q_{k_{2}}-Q_{k_{3}}|)_{\nu} \nonumber\\
    \label{2_2_5_7}\\
    \int_{\{(s,a), r^{'}, s^{'}\}:A1<0,A2<0}(A1)(A2)d(s,a)d{\nu}(s,a)d{\mu_{2}}(r)d{\mu_{3}}(s^{'}) = C_{k_{2}}\int|Q_{k_{2}}-Q_{k_{3}}|d{\nu}(s,a)  =  C_{k_{2}}\mathbb{E}(|Q_{k_{2}}-Q_{k_{3}}|)_{\nu} \nonumber\\
    \label{2_2_5_8}
\end{eqnarray}

We write the last two terms as 

\begin{eqnarray}
     \int_{\{(s,a), r^{'}, s^{'}\}:A1\ge0,A2<0}(A1)(A2)d{\nu}(s,a)d{\mu_{2}}(r)d{\mu_{3}}(s^{'}) =C_{k_{3}}{\epsilon} \label{2_2_5_9}\\
    \int_{\{(s,a), r^{'}, s^{'}\}:A1<0,A2\ge0}(A1)(A2)d{\nu}(s,a)d{\mu_{2}}(r)d{\mu_{3}}(s^{'}) = C_{k_{4}}{\epsilon} \label{2_2_5_10}
\end{eqnarray}

Here $C_{k_{1}}, C_{k_{2}}, C_{k_{4}}$ and $C_{k_{4}}$ are positive constants. Plugging Equations \eqref{2_2_5_7}, \eqref{2_2_5_8}, \eqref{2_2_5_9}, \eqref{2_2_5_10} into Equation \eqref{2_2_5_5}.  

\begin{eqnarray}
    (C_{k_{1}}+C_{k_{2}})\mathbb{E}(|Q_{k_{2}}-Q_{k_{3}}|)_{\nu} -(C_{k_{3}}+C_{k_{4}})\epsilon &\le& \epsilon \label{2_2_5_11}\\
\end{eqnarray}

which implies 

\begin{eqnarray}
    \mathbb{E}(|Q_{k_{2}}-Q_{k_{3}}|)_{\nu} &\le& \left(\frac{1+C_{k_{3}}+C_{k_{4}}}{C_{k_{1}}+C_{k_{2}}}\right)\epsilon \label{2_2_5_12}\\
\end{eqnarray}

Now define $\left(\frac{1+C_{k_{3}}+C_{k_{4}}}{C_{k_{1}}+C_{k_{2}}}\right)=C_{k}$ to get

\begin{eqnarray}
    \mathbb{E}(|Q_{k_{2}}-Q_{k_{3}}|)_{\nu} &\le& C_{k}\epsilon \label{2_2_5_13}\\
\end{eqnarray}

Therefore, we have that if  the number of samples of the data set $X$ of $(s_{i}, a_{i})$ pairs drawn independently from $\nu$ satisfies 

\begin{equation}
  n_{k} \ge 2\left(\frac{2{\eta}\beta_{k}}{\epsilon}\right)^{2},   \label{lem_2_29}
\end{equation}

then we get

\begin{eqnarray}
    \mathbb{E}(|Q_{k2} -Q_{k3}|)_{\nu}  &\le& \frac{\epsilon}{C_{k}} \label{lem_2_31},
\end{eqnarray}

Replacing $\epsilon$ with $\epsilon(C_{k})$ in Equation \eqref{lem_2_29}, we obtain that if 

\begin{equation}
 n_{k} \ge 8\left({(C_{k}{\eta}\beta_{k}})^{2}{\epsilon}^{-2}\right),     \label{lem_2_32}
\end{equation}

we get

\begin{equation}
 \mathbb{E}(|Q_{k2} -Q_{k3}|)_{\nu}  \le \epsilon   \label{lem_2_33}
\end{equation}

which implies

\begin{equation}
 \mathbb{E}(|\epsilon_{k_{3}}|)_{\nu}  \le \epsilon    \label{lem_2_34}
\end{equation}

\end{proof}

\subsection{Proof Of Lemma 4} \label{proof_lem_4}

\begin{proof}
For a given iteration $k$  of Algorithm \ref{algo_1} the optimization problem to be solved in Algorithm \ref{algo_2} is the following 

\begin{equation}
   \mathcal{L}(\theta) = \frac{1}{n} \sum_{i=1}^{n} \left( Q_{\theta}(s_{i},a_{i}) - \left(r(s_{i},a_{i}) + \gamma\max_{a^{'} \in \mathcal{A}}\gamma Q_{k-1}(s^{'},a^{'}) \right) \right)^{2} \label{lem_4_1}
\end{equation}

Here, $Q_{k-1}$ is the estimate of the $Q$ function from the iteration $k-1$ and the state action pairs $(s_{i},a_{i})_{i=\{1,\cdots,n\}}$ have been sampled from a distribution over the state action pairs denoted by $\nu$. Since $\min_{\theta}\mathcal{L}(\theta)$ is a non convex optimization problem we instead solve the equivalent convex problem given by

\begin{eqnarray}
      u_{k}^{*} &=& \argminA_{u}g_{k}(u) =\argminA_{u}|| \sum_{D_{i} \in \tilde{D}}D_{i}X_{k}u_{i}-y_{k} ||^{2}_{2}   \label{lem_4_2} \\
      &&\textit{subject to}   |u|_{1} \le \frac{R_{\max}}{1-\gamma} \label{lem_4_2_1}
\end{eqnarray}

Here, $X_{k} \in \mathbb{R}^{n_{k} \times d}$ is the matrix of sampled state action pairs at iteration $k$, $y_{k} \in \mathbb{R}^{n_{k} \times 1}$ is the vector of target values at iteration $k$. $\tilde{D}$ is the set of diagonal matrices obtained from line \ref{a2_l1} of Algorithm \ref{algo_2} and $u \in \mathbb{R}^{|\tilde{D}d| \times 1}$ (Note that we are treating $u$ as a vector here for notational convenience instead of a matrix as was done in Section \ref{Proposed Algorithm}).

The constraint in Equation \eqref{lem_4_2_1} ensures that the all the co-ordinates of the vector $\sum_{D_{i} \in \tilde{D}}D_{i}X_{k}u_{i}$  are upper bounded by $\frac{R_{max}}{1-\gamma}$ (since all elements of $X_{k}$ are between $0$ and $1$). This ensures that the corresponding neural network represented by Equation \eqref{ReLU_1_2} is also upper bounded by   $\frac{R_{max}}{1-\gamma}$. We use the a projected gradient descent to solve the constrained convex optimization problem which can be written as. 

\begin{eqnarray}
      u_{k}^{*} &=& \argminA_{u : |u|_{1} \le \frac{R_{\max}}{1-\gamma}}g_{k}(u) =\argminA_{u : |u|_{1} \le \frac{R_{\max}}{1-\gamma}}|| \sum_{D_{i} \in \tilde{D}}D_{i}X_{k}u_{i}-y_{k} ||^{2}_{2} \label{lem_4_2_2} 
\end{eqnarray}

From Ang, Andersen(2017). “Continuous Optimization” [Notes]. https://angms.science/doc/CVX  we have that if the step size $\alpha_{k}=\frac{||u^{*}_{k}||_{2}}{L_{k}\sqrt{T_{k}+1}}$, after  $K$ iterations of the projected gradient descent algorithm we obtain

\begin{equation}
(g_{k}(u_{T_{k}})-g_{k}(u^{*})) \le  L_{k} \frac{||u_{k}^{*}||_{2}}{\sqrt{T_{k}+1}}  \label{lem_4_2_3}
\end{equation}

Where $L_{k}$ is the lipschitz constant of $g_{k}(u)$ and $u_{T_{k}}$ is the parameter estimate at step $T_{k}$. 

Therefore if the number of iteration of the projected gradient descent algorithm $T_{k}$  and the step-size $\alpha$ satisfy

\begin{eqnarray}
T_{k} &\ge& L_{k}^{2}||u_{k}^{*}||^{2}_{2}\epsilon^{-2} - 1, \label{lem_4_4}\\
\alpha_{k} &=& \frac{||u^{*}_{k}||_{2}}{L_{k}\sqrt{T_{k}+1}},
\end{eqnarray}

we have 

\begin{equation}
(g_{k}(u_{T_{k}})-g_{k}(u^{*})) \le  \epsilon  \label{lem_4_5}
\end{equation}

Let $(v^{*}_{i},w^{*}_{i})_{i \in (1,\cdots,|\tilde{D}|)}$, $(v^{T_{k}}_{i},w^{T_{k}}_{i})_{i \in (1,\cdots,|\tilde{D}|)}$ be defined as 
\begin{eqnarray}
 (v^{*}_{i},w^{*}_{i})_{i \in (1,\cdots,|\tilde{D}|)} = \psi_{i}^{'}(u^{*}_{i})_{i \in (1,\cdots,|\tilde{D}|)} \label{lem_4_7}\\
 (v^{T_{k}}_{i},w^{T_{k}}_{i})_{i \in (1,\cdots,|\tilde{D}|)} = \psi_{i}^{'}(u^{T_{k}}_{i})_{i \in (1,\cdots,|\tilde{D}|)} \label{lem_4_8}
\end{eqnarray}

where $\psi^{'}$ is defined in Equation \eqref{ReLU_3_1}.

Further, we define  $\theta_{|\tilde{D}|}^{*}$ and $\theta^{T_{k}}$ as 
\begin{eqnarray}
\theta_{|\tilde{D}|}^{*} = \psi(v^{*}_{i},w^{*}_{i})_{i \in (1,\cdots,|\tilde{D}|)}   \label{lem_4_9}\\
\theta^{T_{k}} = \psi(v^{T_{k}}_{i},w^{T_{k}}_{i})_{i \in (1,\cdots,|\tilde{D}|)}   \label{lem_4_10}
\end{eqnarray}

 where  $\psi$ is defined in Equation \eqref{ReLU_2_1},  $\theta_{|\tilde{D}|}^{*} = \argminA_{\theta}\mathcal{L}_{|\tilde{D}|}(\theta)$ for $\mathcal{L}_{|\tilde{D}|}(\theta)$ defined in Appendix \ref{cones_apdx}.

Since  $(g(u_{T_{k}})-g(u^{*})) \le \epsilon$, then by Lemma \ref{sup_lem_3}, we have 

\begin{equation}
     \mathcal{L}_{|\tilde{D}|}(\theta^{T_{k}}) -\mathcal{L}_{|\tilde{D}|}(\theta_{|\tilde{D}|}^{*}) \le \epsilon  \label{lem_4_11}
\end{equation}

Note that $\mathcal{L}_{|\tilde{D}|}(\theta^{T_{k}}) -\mathcal{L}_{|\tilde{D}|}(\theta_{|\tilde{D}|}^{*})$ is a constant value. Thus we can always find constant $C^{'}_{k}$ such that

\begin{equation}
     C_{k}^{'}|\theta^{T_{k}}-\theta_{|\tilde{D}|}^{*}|_{1}   \le  \mathcal{L}_{|\tilde{D}|}(\theta^{T_{k}}) -\mathcal{L}_{|\tilde{D}|}(\theta_{|\tilde{D}|}^{*})  \label{lem_4_14}
\end{equation}

\begin{equation}
     |\theta^{T_{k}}-\theta_{|\tilde{D}|}^{*}|_{1}   \le  \frac{\mathcal{L}(\theta^{T_{k}}) -\mathcal{L}(\theta^{*})}{C_{k}^{'}}  \label{lem_4_15}\\
\end{equation}

Therefore if we have 

\begin{eqnarray}
T_{k} &\ge& L_{k}^{2}||u_{k}^{*}||^{2}_{2}\epsilon^{-2} - 1, \label{lem_4_16}\\
\alpha_{k} &=& \frac{||u^{*}_{k}||_{2}}{L_{k}\sqrt{T_{k}+1}},
\end{eqnarray}

then we have 

\begin{equation}
     |\theta^{T_{k}}-\theta^{*}|_{1}   \le \frac{\epsilon}{C_{k}^{'}}  \label{lem_4_17}\\
\end{equation}

which according to Equation \eqref{lem_4_15} implies that

\begin{equation}
      C_{k}^{'}|\theta^{T_{k}}-\theta_{|\tilde{D}|}^{*}|_{1}   \le  \mathcal{L}_{|\tilde{D}|}(\theta^{T_{k}}) -\mathcal{L}_{|\tilde{D}|}(\theta_{|\tilde{D}|}^{*})   \le \epsilon \label{lem_4_17_1}\\
\end{equation}

Dividing Equation \eqref{lem_4_17_1} by $C_{k}^{'}$ we get 

\begin{equation}
      |\theta^{T_{k}}-\theta_{|\tilde{D}|}^{*}|_{1}   \le  \frac{\mathcal{L}_{|\tilde{D}|}(\theta^{T_{k}}) -\mathcal{L}_{|\tilde{D}|}(\theta_{|\tilde{D}|}^{*})}{C_{k}^{'}}   \le \frac{\epsilon}{C_{k}^{'}} \label{lem_4_17_2}\\
\end{equation}

Which implies 

\begin{equation}
      |\theta^{T_{k}}-\theta_{|\tilde{D}|}^{*}|_{1}  \le \frac{\epsilon}{C_{k}^{'}} \label{lem_4_17_3}\\
\end{equation}

Assuming $\epsilon$ is small enough such that $\frac{\epsilon}{C_{k}^{'}} < 1$ from lemma \ref{sup_lem_5}, this implies that there exists an $l_{k}>0$ such that    

\begin{eqnarray}
     |Q_{\theta^{T_{k}}}(s,a)-Q_{\theta_{|\tilde{D}|}^{*}}(s,a)| &\le&  l_{k}|\theta^{T_{k}}-\theta_{|\tilde{D}|}^{*}|_{1}   \label{lem_4_19}\\
                                         &\le&  \frac{l_{k}\epsilon}{C_{k}^{'}} \label{lem_4_20}
\end{eqnarray}
for all $(s,a) \in \mathcal{S}\times\mathcal{A}$. Equation \eqref{lem_4_20} implies that if

\begin{eqnarray}
T_{k}  &\ge& L_{k}^{2}||u_{k}^{*}||^{2}_{2}\epsilon^{-2} - 1,\\
\alpha &=& \frac{||u^{*}_{k}||_{2}}{L_{k}\sqrt{T_{k}+1}},
\end{eqnarray}

then we have

\begin{eqnarray}
     \mathbb{E}(|Q_{\theta^{T_{k}}}(s,a)-Q_{\theta_{|\tilde{D}|}^{*}}(s,a)|)_{\nu} &\le& \frac{l_{k}\epsilon}{C_{k}^{'}} \label{lem_4_21}
\end{eqnarray}

By definition in section \ref{thm proof} $Q_{k}$ is our estimate of the $Q$ function at the $k^{th}$ iteration of Algorithm $1$ and thus we have  $Q_{\theta^{T_{k}}}=Q_{k}$ which implies that

\begin{eqnarray}
     \mathbb{E}(|Q_{k}(s,a)-Q_{\theta_{\tilde{D}}^{*}}(s,a)|)_{\nu} &\le& \frac{l_{k}\epsilon}{C_{k}^{'}} \label{lem_4_22}
\end{eqnarray}

 If we replace $\epsilon$ by $\frac{C^{'}_{k}\epsilon}{l_{k}}$ in Equation \eqref{lem_4_21}, we get that if 

\begin{eqnarray}
T_{k} &\ge& \left({\frac{C^{'}_{k}\epsilon}{l_{k}}}\right)^{-2}L_{k}^{2}||u_{k}^{*}||^{2}_{2} - 1, \label{lem_4_23}\\
\alpha &=& \frac{||u^{*}_{k}||_{2}}{L_{k}\sqrt{T_{k}+1}},
\end{eqnarray}

we have

\begin{eqnarray}
     \mathbb{E}(|Q_{k}(s,a)-Q_{\theta_{\tilde{D}}^{*}}(s,a)|)_{\nu} &\le& \epsilon \label{lem_4_24}
\end{eqnarray}

From Assumption \ref{assump_2}, we have that 

\begin{eqnarray}
     \mathbb{E}(|Q_{\theta^{*}}(s,a)-Q_{\theta_{\tilde{D}}^{*}}(s,a)|)_{\nu} &\le& \epsilon_{|\tilde{D}|} \label{lem_4_25}
\end{eqnarray}

where $\theta^{*} = \argminA_{\theta \in \Theta}\mathcal{L}(\theta)$ and by definition of $Q_{k_3}$ in Definition \ref{def_6}, we have that  $Q_{k_3}=Q_{\theta^{*}}$. Therefore if we have 

\begin{eqnarray}
T_{k} &\ge& \left({\frac{C^{'}_{k}\epsilon}{l_{k}}}\right)^{-2}L_{k}^{2}||u_{k}^{*}||^{2}_{2} - 1, \label{lem_4_26}\\
\alpha &=& \frac{||u^{*}_{k}||_{2}}{L_{k}\sqrt{T_{k}+1}},
\end{eqnarray}

we have 

\begin{eqnarray}
     \mathbb{E}(|Q_{k}(s,a)-Q_{k_3}(s,a)|)_{\nu} &\le&  \mathbb{E}(|Q_{k}(s,a)-Q_{\theta_{\tilde{D}}^{*}}(s,a)|)_{\nu}   + \mathbb{E}(|Q_{k_3}(s,a)-Q_{\theta_{\tilde{D}}^{*}}(s,a)|)_{\nu}   \label{lem_4_27}\\
                                                 &\le&  \epsilon + \epsilon_{|\tilde{D}|} \label{lem_4_28}
\end{eqnarray}

\end{proof}

\end{document}